\documentclass[11pt]{article}

\usepackage{geometry}
 \usepackage[colorlinks = true,
            linkcolor = blue,
            urlcolor  = blue,
            citecolor = blue,
            anchorcolor = blue]{hyperref}

\usepackage{dsfont}
\usepackage{natbib}
\usepackage{soul}
\usepackage{xfrac}
\usepackage{amsmath}
\usepackage{tikz}
\usepackage{mathtools}
\usepackage{nccmath}
\usepackage{amsthm}
\usepackage{amssymb}
\usepackage{hyperref}
\usepackage{xcolor}
\usepackage[capitalise]{cleveref}
\usepackage{xpatch}
\usepackage{caption}
\usepackage{algorithm}
\usepackage[noend]{algpseudocode}
\usepackage{newfloat}
\usepackage[inline]{enumitem}
\usepackage{forest,adjustbox}
\usepackage{csquotes}
\usepackage{setspace}
\usepackage{etoolbox}
\AtBeginEnvironment{quote}{\par\singlespacing\small}
\usepackage{subcaption}
\usepackage{comment}     
\usetikzlibrary{calc}  
\usetikzlibrary{angles,quotes}
\usepackage{tikz-3dplot}

\usepackage[most]{tcolorbox}
\tcbset{colback=gray!10!white, colframe=black, boxrule=0.5pt, arc=2mm, left=4pt, right=4pt, top=4pt, bottom=4pt}

\newcommand{\dist}{\mathrm{dist}_X}

\newcommand{\OPT}{\mathrm{OPT}}

\newcommand{\diam}{\mathrm{diam}}
\newcommand{\vvv}{\overrightarrow}
\newcommand{\argmin}{\mathrm{argmin}}
\newcommand{\argmax}{\mathrm{argmax}}

\newtheorem{theorem}{Theorem}
\newtheorem*{theorem*}{Theorem}
\newtheorem{lemma}[theorem]{Lemma}
\newtheorem{proposition}[theorem]{Proposition}
\newtheorem*{proposition*}{Proposition}

\newtheorem{claim}[theorem]{Claim}
\newtheorem{claim*}{Claim}
\newtheorem*{example*}{Example}
\newtheorem{example}[theorem]{Example}

\newtheorem{definition}[theorem]{Definition}
\newtheorem{notation}[theorem]{Notation}
\newtheorem{remark}[theorem]{Remark}
\newtheorem*{remark*}{Remark}
\newtheorem*{remarks*}{Remarks}
\newtheorem{question}[theorem]{Open Question}
\newtheorem{observation}[theorem]{Observation}

\newcommand{\cA}{\mathcal{A}}

\newcommand{\cR}{\mathcal{R}}

\newcommand{\cK}{\mathcal{K}}

\newcommand{\RR}{\mathbb{R}}

\newcommand{\NN}{\mathbb{N}}

\hypersetup{
   breaklinks=true,   % splits links across lines
   colorlinks=true,   % displays links as colored text instead of blocks
}

  \author{Shay Moran\footnote{Departments of Mathematics, Computer Science, and Data and Decision Sciences, Technion and Google Research.} \and Elizaveta Nesterova\footnote{Departments of Mathematics, Technion.}}
\title{Reconstruction and Secrecy under Approximate Distance Queries}
\begin{document}

\maketitle

\begin{abstract}
  Consider the task of locating an unknown target point using approximate distance queries: in each round, a reconstructor selects a reference point and receives a noisy version of its distance to the target. This problem arises naturally in various contexts—ranging from localization in GPS and sensor networks to privacy-aware data access—and spans a wide variety of metric spaces. It is relevant from the perspective of both the reconstructor (seeking accurate recovery) and the responder (aiming to limit information disclosure, e.g., for privacy or security reasons). We study this reconstruction game through a learning-theoretic lens, focusing on the rate and limits of the best possible reconstruction error. Our first result provides a tight geometric characterization of the optimal error in terms of the Chebyshev radius, a classical concept from geometry. This characterization applies to all compact metric spaces (in fact, even to all totally bounded spaces) and yields explicit formulas for natural metric spaces. Our second result addresses the asymptotic behavior of reconstruction, distinguishing between pseudo-finite spaces—where the optimal error is attained after finitely many queries—and spaces where the approximation curve exhibits a nontrivial decay. We characterize pseudo-finiteness for convex Euclidean spaces.
\end{abstract}

\section{Introduction}
\label{sec:intro}

In the \emph{reconstruction game}, a reconstructor seeks to locate an unknown point \( x^\star \) in a metric space \( (X, \dist) \) using a sequence of approximate distance queries. In each round, the reconstructor selects a query point \( q_t \in X \) and receives a response \( \hat{d}_t \) that approximates the true distance \( \dist(q_t, x^\star) \). The approximation is controlled by two error parameters: \( \epsilon \ge 0 \), which bounds the multiplicative error, and \( \delta \ge 0 \), which bounds the additive error. Specifically, the response satisfies \( \hat{d}_t =_{\epsilon,\delta} \dist(q_t, x^\star) \), where
\[
x =_{\epsilon, \delta} y \quad \text{means that} \quad x \leq (1+\epsilon)y + \delta \quad \text{and} \quad y \leq (1+\epsilon)x + \delta.
\]
After a bounded number of such queries, the reconstructor outputs a guess \( \hat{x} \in X \) and aims to minimize the reconstruction error \( \dist(\hat{x}, x^\star) \).

This simple game arises in a wide range of natural scenarios. In \emph{privacy-preserving data analysis}, it models the trade-off between utility and privacy: a responder must answer queries while protecting sensitive data, as in the foundational work of \citep{DinurN03} that initiated the study of differential privacy\footnote{The reconstruction model studied by Dinur and Nissim uses counting (or linear) queries. Still, it is essentially equivalent to our model with distance queries. We elaborate on this connection at \cref{ex:DinurNissim_equivalence}.}. In \emph{computational geometry}, related questions arise when inferring geometric structures from noisy measurements~\citep{DCGHandbook-18}. In \emph{remote sensing}, satellites and sensors reconstruct physical information—such as terrain or atmospheric properties—from indirect and error-prone signals~\citep{Twomey1977}. Similar structures also appear in \emph{learning theory}: for instance, hypothesis selection and distribution learning via statistical queries can be framed as reconstruction problems over suitable metric spaces.
% ;one basic example is Yatracos's seminal work on distribution learning~\citep{Yatracos85}, which studies the problem of learning a distribution from a class \( \mathcal{H} \) with respect to the total variation metric. His approach effectively reduces the problem to reconstruction in a metric space induced by the so-called \emph{Yatracos sets}—a family of sets determined by the class~\(\mathcal{H} \).
%These examples demonstrate the generality of the reconstruction game, which captures a broad class of interactive information acquisition tasks under uncertainty.

The reconstruction game captures a natural tension between two objectives: recovering hidden information and limiting what can be revealed. From the reconstructor’s perspective, the task is to approximate an unknown point from noisy distance measurements. This challenge arises in a variety of applications, including navigation, search-and-rescue, and remote sensing, where inference must be made under uncertainty. On the other side, the responder may wish to share useful information while restricting what can be inferred—whether for reasons of privacy, security, or resource constraints. This interplay between noisy access and limited disclosure makes the model relevant across several domains.

While this framework treats the reconstructor and the responder symmetrically, assigning equal roles to both, our technical results focus on understanding the limits of what the reconstructor can achieve. We present two main contributions:

\paragraph{Limit of optimal reconstruction (Theorem~\ref{thm:Jung}).}
We characterize the optimal approximation error that the reconstructor can guarantee in the limit, as the number of queries tends to infinity. This error depends on the metric space and the approximation parameters \( \epsilon \) and \( \delta \), and our characterization applies to all totally bounded metric spaces. The result is expressed in terms of a classical geometric quantity: the Chebyshev radius. This limiting error plays a role analogous to the Bayes optimal error in statistical learning: it captures the best achievable performance in the presence of noise, regardless of the specific strategy or number of queries. Our result provides a geometric characterization and interpretation of this optimum in the context of the reconstruction game.

\paragraph{Pseudo-finite Spaces (Theorem~\ref{thm:pseudo-finiteness}).}
Beyond the limiting error, a central question concerns the rate at which this optimum is approached as a function of the number of queries. This question is inherently rich and depends delicately on the geometry of the space. It is analogous to the study of \emph{learning curves} in statistical learning theory, which quantify how the performance of a learner improves with more data. We initiate the study of this question in our setting by identifying and analyzing a fundamental distinction between \emph{pseudo-finite} spaces—where finitely many queries suffice to reach the optimum—and spaces where convergence is gradual. We show that this notion is already subtle and nontrivial, and we provide a characterization of pseudo-finiteness for convex Euclidean spaces.

%\vspace{-5.975ex}
\subsection{Problem Setup and Main Results}
%\vspace{-5ex}

We now define the \emph{reconstruction game}, a formal interaction between two players: a \emph{reconstructor} (RC) and a \emph{responder} (RSP). The game takes place in a metric space \( (X, \dist) \), where \( X \) is the domain and \( \dist: X \times X \to \mathbb{R}_{\ge 0} \) is a distance function.
The interaction is governed by two error parameters: \( \epsilon \ge 0 \), which controls multiplicative distortion, and \( \delta \ge 0 \), which controls additive distortion. 
%These parameters model various sources of inaccuracy in the responses—ranging from measurement noise to privacy-preserving perturbations.

\begin{tcolorbox}[title=The Reconstruction Game]
The game is parameterized by \( \epsilon, \delta \ge 0 \), and is played over \( T \) rounds in a metric space \( (X, \dist) \).

Each round \( t = 1, \dots, T \) proceeds as follows:

\begin{enumerate}
    \item The reconstructor submits a query point \( q_t \in X \).
    \item The responder returns a value \( \hat{d}_t \), which approximates the true distance to some secret point. 
    \end{enumerate}
The responder must ensure that all answers given in the game remain jointly consistent with at least one point \( x^\star\). That is: 
\((\exists x^\star \in X)(\forall t\leq T):\hat{d}_t =_{\epsilon,\delta} \dist(q_t, x^\star),\) where $x=_{\epsilon,\delta} y$ means that
    \(\hat{d}_t \leq (1+\epsilon)\dist(q_t, x^\star) + \delta\) and \(\dist(q_t, x^\star) \leq (1+\epsilon)\hat{d}_t + \delta.\)

\medskip

At the end of the game, the reconstructor outputs a final guess \( \hat{x}_T \in X \).
The reconstruction error is defined as the worst-case distance to a consistent point:
\(
\sup_{x^\star}\dist(\hat{x}_T, x^\star),
\)
where $x^\star$ ranges over all consistent points. 
%The reconstructor aims to minimize this quantity, while the responder aims to maximize it.
\end{tcolorbox}
% \begin{enumerate}
%     \item The reconstructor submits a query point \( q_t \in X \).
%     \item The responder returns a value \( \hat{d}_t \), which approximates the true distance to some secret point \( x^\star \in X \). 
%     The response must satisfy
%     \(
%     \hat{d}_t =_{\epsilon,\delta} \dist(q_t, x^\star),
%     \)
%     meaning that
%     \(
%     \hat{d}_t \leq (1+\epsilon)\dist(q_t, x^\star) + \delta \quad \text{and} \quad
%     \dist(q_t, x^\star) \leq (1+\epsilon)\hat{d}_t + \delta.
%     \)
%     The responder must ensure that all answers given so far remain jointly consistent with at least one point \( x^\star \in X \).
% \end{enumerate}

% At the end of the game, the reconstructor outputs a final guess \( \hat{x}_T \in X \). The reconstruction error is defined as the worst-case distance to a consistent point:
% \[
% \sup \left\{ \dist(\hat{x}_T, x^\star) \mid x^\star \in X \text{ consistent with all } (q_t, \hat{d}_t) \right\}.
% \]
% The reconstructor aims to minimize this quantity, while the responder aims to maximize it.
% \end{tcolorbox}

The reconstruction game studied in this work generalizes the task of determining the \emph{sequential metric dimension} (SMD), which was originally introduced in the \emph{noiseless} setting for graphs by~\citet{seager2013sequential}. The SMD captures the minimum number of exact distance queries required to identify an unknown point exactly, and has been studied in finite metric spaces induced by graphs~\citep{bensmail2020sequential,SeqMetricDimRand21,MetricDimensionSurvey2023}, with particular emphasis on the gap between sequential and static metric dimension.\footnote{The static metric dimension is the minimum number of reference points needed to uniquely determine any point in the space based on its distances to those references. It corresponds to the non-adaptive variant of our setting, where all queries are fixed in advance.}
%In contrast, our work considers general metric spaces and allows for approximate responses, focusing on the fundamental limitations imposed by noise. Bridging these two lines of work—by extending SMD techniques to the approximate setting—remains a promising direction for future research.}

As another example, the counting-query model introduced by \citet{DinurN03} in their foundational work can also be naturally viewed as a special case of our reconstruction game on the Boolean cube endowed with the Hamming metric:

\begin{example}[From counting queries to distance queries]
\label{ex:DinurNissim_equivalence}
In the counting-query model \citep{DinurN03}, the dataset is a binary vector
\(D=(d_1,\ldots,d_n)\in\{0,1\}^n\).
At round \(t=1,\ldots,T\), the reconstructor chooses a subset \(q_t\subseteq[n]\) and receives
\[
a_t \;=\; \sum_{i\in q_t} d_i \;+\; \eta_t,
\qquad |\eta_t|\le \delta.
\]
This game is not syntactically a metric-distance game, yet it is \emph{equivalent} to our distance-query model on the Boolean cube with the Hamming metric, in the 
sense that counting queries and Hamming-distance queries simulate each other with at most a two-query overhead per round.
Full details of the simulation appear in Appendix~\ref{app:examples}, Example~\ref{ex:app_DinurNissim_equivalence}.
\end{example}

\paragraph{A Priori vs.\ A Posteriori Responder.}
There are two natural variants of the reconstruction game, which differ in when the responder commits to the secret point.

In the \emph{a priori} version, the responder selects a secret point \( x^\star \in X \) at the beginning of the game and must answer all queries consistently with that fixed point.
In contrast, the \emph{a posteriori} version (which we adopt) allows the responder to wait \textcolor{black}{until the end of the game when the reconstructor selects her guess $\hat{x}_T$,} before selecting the secret point \( x^\star \). 

Note that for deterministic reconstructors, the a priori and a posteriori models are equivalent: any a posteriori responder can be simulated by an a priori one, simply by anticipating all queries of the reconstructor and precomputing a worst-case consistent point in advance. “Deterministic” here refers only to the absence of internal randomness and does not restrict adaptivity; a deterministic reconstructor may choose each query based on the entire interaction so far, whereas “non-adaptive” denotes the special case in which all queries are fixed in advance.

\begin{remark*}
 We define the game using the a posteriori model because our results focus on the capabilities of the reconstructor in the \emph{worst-case} setting. From this viewpoint, the most meaningful formulation is one that allows the responder maximal flexibility, making the task of reconstruction as difficult as possible.
\end{remark*}

\subsection{Optimal Reconstruction Distance}
At each point in the game, the sequence of query--response pairs received so far determines a \emph{feasible region}—the set of points in \( X \) that are consistent with all previous answers under the error model. The size and geometry of this region reflect the remaining uncertainty about the secret point. From the reconstructor’s perspective, the goal is to make this region as small as possible, ideally identifying a point that is close to every element in it.
We measure the performance of a reconstruction strategy by the worst-case distance between the output \( \hat{x}_T \) and any point in the feasible region. The key quantity we study is the optimal worst-case guarantee achievable by the reconstructor after \( T \) queries, denoted by
\begin{equation}
\label{eq:optimal}
  \OPT_X(T,\epsilon,\delta) := \adjustlimits \inf_{\mathrm{RC}} \sup_{\mathrm{RSP}} \sup_{x \in \Phi_T} \dist(\hat{x}_T, x).
\end{equation}
Here, \( \Phi_T \subseteq X \) is the \emph{feasible region}—the set of points that remain consistent with the transcript (i.e., the sequence of queries and responses) \( \{(q_t, \hat{d}_t)\}_{t=1}^T \)
\begin{equation}
\label{eq:feasible_region}
\Phi(\{q_i, \hat{d}_i\}_{i=1}^T) := \left\{ x \in X \mid \text{ for all } 1 \le i \le T \colon \,\,  \dist(x, q_i) =_{\epsilon,\delta} \hat{d}_i
% \begin{aligned}
%     &\dist(x, q_i) \le (1+\epsilon)r_i + \delta,\\
%     &r_i \le (1+\epsilon)\dist(x, q_i) + \delta
%   \end{aligned}
  \right\}.
\end{equation}
The infimum ranges over all strategies employed by the reconstructor, and each such strategy is evaluated in the worst case: against the most adversarial responder strategy (subject to consistency), and with respect to the most distant feasible point. 
\textcolor{black}{For randomized reconstructors, we interpret \Cref{eq:optimal} by replacing \( \dist(\hat{x}_T, x) \) with \( \mathbb{E}[\dist(\hat{x}_T, x)] \), where the expectation is over the internal randomness of the reconstructor. For simplicity of presentation, we assume the reconstructor is deterministic; however, all of our results and proofs extend to the randomized setting.}

Much of our focus will be on understanding how this function behaves as \( T \to \infty \), and how it depends on the geometry of the underlying metric space \( X \). 
Our first main result concerns the asymptotic quantity
\[
\OPT_X(\epsilon, \delta) := \lim_{T \to \infty} \OPT_X(T, \epsilon, \delta),
\]
which captures the best reconstruction error the reconstructor can guarantee in the limit, as the number of queries grows\footnote{As shown in \cref{app:feasible}, specifically in \cref{cor:monotone}, the function \( \OPT_X(T, \epsilon, \delta) \) is monotonically non-increasing in \( T \), so this limit always exists.}.

\paragraph{Chebyshev Radius.} 
To characterize \( \OPT_X(\epsilon, \delta) \), we rely on a classical geometric quantity called the \emph{Chebyshev radius}, which captures how well a set can be enclosed by a ball. Let \( (X, \dist) \) be a metric space, and let \( \alpha > 0 \) be a parameter. For any subset \( S \subseteq X \), we denote its diameter by \( \diam(S) := \sup_{x,y \in S} \dist(x, y) \). The \emph{Chebyshev radius} of \( S \), denoted \( r(S) \), is defined as
\[
r(S) := \adjustlimits \inf_{x \in X} \sup_{y \in S} \left[\dist(x, y) \right],
\]
that is, the smallest radius for which some ball centered in \( X \) contains all of \( S \).

We will also rely on the following quantity that captures the worst-case relationship between sets of diameter at most \( \alpha \) and the radius of their smallest enclosing ball:
\[
\mathtt{e}_X(\alpha) := \sup_{S:\,\diam(S) \leq \alpha} r(S).
\]
We call this function \emph{diameter-radius profile}. 
For example, in Euclidean space \( (\mathbb{R}^n, \ell_2) \), it is known that for all \( \alpha > 0 \),
\(
\mathtt{e}_X(\alpha) = \sqrt{\frac{n}{2(n+1)}} \cdot \alpha,
\)
as shown, for instance, in \cite{Blumenthal1953}'s monograph.

%The function \( \mathtt{e}_X(\alpha) \) is tightly linked to the Jung\footnote{The Jung constant is usually defined for normed vector spaces as \( \Jung(X,\alpha) := \sup\{ r(S)/\alpha : \diam(S) = \alpha \} \), and in such spaces it is independent of \( \alpha \). This differs from \( \mathtt{e}_X(\alpha)/\alpha \), which takes the supremum over sets of diameter \emph{at most} \( \alpha \). In many natural metric spaces (such as normed spaces), the two definitions agree, since any set of smaller diameter can be extended to one of diameter exactly \( \alpha \).} constant, which is essentially the normalized ratio \( \mathtt{e}_X(\alpha)/\alpha \). This geometric invariant was first studied by H.~Jung in his 1901 paper \citep{jung1901}; for additional connections between Chebyshev centers and the Jung constant, see Blumenthal’s monograph \citep{Blumenthal1953}.

Before stating our first main result, we recall a standard notion from metric geometry. A metric space \( (X, \dist) \) is said to be \emph{totally bounded} if for every \( r > 0 \), there exists a finite cover of \( X \) by balls of radius \( r \). This condition is a common weakening of compactness that still ensures many desirable finiteness properties. As we will see in Section~\ref{sec:examples}, this assumption is necessary for the theorem’s conclusion; without it, the game can trivialize, allowing the responder to force an approximation error equal to the space’s diameter.

%lifting it allows for exotic metric spaces in which it is not clear whether anything non-trivial can be said about the value of the optimal reconstruction error.

\begin{tcolorbox}[title = Tight Error via Chebyshev Radius]
\begin{theorem}\label{thm:Jung}
Let \( X \) be a totally bounded metric space. Then, for any \( \epsilon, \delta \ge 0 \),
\[
\OPT_X(\epsilon,\delta) = \mathtt{e}_X\bigl((2+\epsilon)\delta\bigr).
\]
Moreover, if the distance \( (2+\epsilon)\delta \) is realized in \( X \), i.e., there exists a pair of points at this distance, then
\[
\frac{1}{2}(2+\epsilon)\delta \leq \OPT_X(\epsilon,\delta) \leq (2+\epsilon)\delta.
\]
\end{theorem}
\end{tcolorbox}

This result expresses the limiting reconstruction error in terms of the function \( \mathtt{e}_X(\cdot) \), which captures the worst-case Chebyshev radius over sets of bounded diameter. While the definition of \( \mathtt{e}_X((2+\epsilon)\delta) \) may seem somewhat cryptic at first glance, it is often closely tied to the scale of noise introduced by the responder. Specifically, in many natural spaces, it holds that
\[
\OPT_X(\epsilon, \delta) = \Theta\bigl((2+\epsilon)\delta\bigr).
\]
This follows from the next general observation, which bounds the ratio between the Chebyshev radius and the diameter of a set in any metric space.

\begin{observation}
\label{obs:jung-bounds}
In any metric space \( (X, \dist) \) and every \( \alpha > 0 \) which is realized as a distance in the space,
\(
\frac{1}{2}\alpha \le \mathtt{e}_X(\alpha) \le \alpha.
\)
The upper bound follows because any set of diameter \( \alpha \) can be trivially enclosed in a ball of radius~\(\alpha \). The lower bound holds because no ball of radius \( r < \alpha/2 \) can contain two points at distance \( \alpha \).
\end{observation}
The bounds \( \alpha/2 \) and \( \alpha \) are tight: they are attained by natural totally bounded metric spaces, as we will demonstrate through examples in Section~\ref{sec:examples}.

%\medskip

\subsection{Excess Reconstruction Error}
From a learning-theoretic perspective, the limiting error \( \OPT_X(\epsilon,\delta) \) plays a role analogous to the \emph{Bayes optimal} error in statistical learning: it represents the best achievable performance under the constraints of the model. This motivates the study of the \emph{excess reconstruction error}—the difference between the error achieved after \( T \) queries and this asymptotic optimum:
\(
\OPT_X(T, \epsilon, \delta) - \OPT_X(\epsilon, \delta).
\)
Understanding the rate at which this quantity decays as \( T \to \infty \) is a natural next step.

This question presents significant challenges and is highly dependent on the geometry of the underlying space. As a first step in this direction, our second main result focuses on a basic dichotomy: between spaces where convergence to the optimal error is trivial—i.e., achieved after finitely many queries—and all others. We formalize this notion through the following definition:
\begin{definition}[Pseudo‑finite Spaces]\label{def:pseudofinite}
A metric space $(X,\dist)$ is said to be \emph{$(\epsilon, \delta)$‑pseudo‑finite} if there exists a finite constant \( T_{X,\epsilon,\delta} < \infty \) such that
\[
\OPT_X(T, \epsilon,\delta) = \OPT_X(\epsilon,\delta)
\quad\text{for all } T \ge T_{X,\epsilon,\delta}.
\]
\end{definition}
It is easy to see that any finite metric space is $(\epsilon,\delta)$‑pseudo‑finite for all values of $\epsilon, \delta \ge 0$: the reconstructor can query every point in the space, and no additional information can be obtained once all points have been queried.
Another example of pseudo-finiteness is provided by finite-dimensional Euclidean spaces. The space $\mathbb{R}^n$ is $(0,0)$‑pseudo‑finite\footnote{Note that noise plays an important role in this example: 
\(\RR^n\) is \emph{not} pseudo-finite whenever the noise parameters are nonzero and \(n \ge 2\).} since the reconstructor can determine the exact location of the secret by querying $n+1$ affinely independent points (see, e.g.\citep{MetricDimensionSurvey2023}). In contrast, we will see in the next section an example of a totally bounded metric space that is not $(0,0)$‑pseudo‑finite.

We now turn our attention to Euclidean spaces. Naturally, we begin with the simplest case: the real line. Despite its simplicity, the real line exhibits a nuanced pseudo-finiteness behavior that depends on the error parameters. In particular, pseudo-finiteness holds when there is no multiplicative noise, but breaks down as soon as any multiplicative distortion is allowed:

\begin{proposition}[Pseudo-finiteness of the real line]
\label{prop:real-line-pseudofinite}
Let \( X = [0,1] \subseteq \mathbb{R} \) equipped with the standard Euclidean metric. Then:
(i) For every \( \delta \ge 0 \), the space \( X \) is \( (0,\delta) \)-pseudo-finite. (ii) For every \( \epsilon > 0 \) and every \( \delta \ge 0 \), the space \( X \) is not \( (\epsilon,\delta) \)-pseudo-finite.
% \begin{enumerate}[label=(\roman*)]
%     \item 
%     \item 
% \end{enumerate}
\end{proposition}
\noindent
This proposition follows from our general result below (Theorem~\ref{thm:pseudo-finiteness}), but can also be derived more directly in this special case. When \( \epsilon = 0 \), the reconstructor can query one of the endpoints \( q_1 \in \{0,1\} \); the response confines the secret to an interval of length \( 2\delta \), and outputting its midpoint yields an error of at most \( \delta \), which is optimal\footnote{The optimality of the error \(\delta\) on the interval, for sufficiently small \(\delta\), follows from the fact that \(\mathtt e_{[0,1]}(\delta) = \delta\) together with Theorem~\ref{thm:Jung}.}. When \( \epsilon > 0 \), the responder can use a binary-search-like strategy to ensure that the feasible region always contains an interval of length strictly greater than \( (2+\epsilon)\delta \), thereby preventing the reconstructor from reaching the optimum in finitely many steps.
% These claims correspond to Example~\ref{ex:real-line}, and are further discussed in Section~\ref{sec:examples}.

How about higher-dimensional Euclidean spaces—do they exhibit the same behavior as the real line with respect to pseudo-finiteness? Our second main result addresses this question for the class of convex subsets of Euclidean space. To state it, we recall that the \emph{dimension} of a convex set~\(X \subseteq \mathbb{R}^n \) refers to the dimension of its affine span, i.e., the smallest affine subspace containing~\(X \). In higher dimensions, this nuanced behavior disappears: convex subsets of \( \mathbb{R}^n \) with dimension at least two are never pseudo-finite, regardless of the values of \( \epsilon \) and \( \delta \), as long as they are sufficiently small compared to the diameter.

% How about higher-dimensional Euclidean spaces—do they exhibit the same behavior as the real line with respect to pseudo-finiteness? Our second main result addresses this question for the class of convex subsets of Euclidean space. To state it, we recall that the \emph{dimension} of a convex set~\(X \subseteq \mathbb{R}^n \) refers to the dimension of its affine span, i.e., the smallest affine subspace containing~\(X \). The result shows that the nuanced behaviour which occurs on the real line does not occur in higher dimensions in which convex spaces are never pseudo-finite (for no eps delta)

%pseudo-finiteness occurs only in a very restricted setting: convex sets of dimension one with zero multiplicative error. In all other cases—positive dimension and small but non-zero noise—pseudo-finiteness does not hold.

\begin{tcolorbox}[title={Pseudo‑Finiteness in Convex Euclidean Spaces}]
\begin{theorem}\label{thm:pseudo-finiteness}
Let \(X \subset \mathbb{R}^n\) be a bounded convex set equipped with the Euclidean metric such that \(\dim X > 0\) and let \(\epsilon \ge 0\). Then, for all sufficiently small \(\delta > 0\), the space \(X\) is \underline{\emph{not}} \((\epsilon, \delta)\)-pseudo‑finite, except in the case where \(\epsilon = 0\) and \(\dim X = 1\).
\end{theorem}
\end{tcolorbox}

The proof of this result is surprisingly delicate. At a high level, one might expect that a responder could simply inject random noise into the true distances, thereby ensuring that the reconstructor improves only gradually over time. However, such a strategy does not suffice to rule out pseudo-finiteness: to do so, one must ensure that for every reconstructor strategy, the reconstruction error remains strictly larger than the optimal limit for any finite number of queries. This requires carefully calibrated noise that not only misleads the reconstructor but also guarantees that the resulting feasible region strictly contains a set of points forming an extremal body—one that achieves the maximal Chebyshev radius under a bounded diameter constraint. 

In fact, the lower bound on \( \OPT_X(\epsilon, \delta) \) established in Theorem~\ref{thm:Jung} is implicitly used in proving Theorem~\ref{thm:pseudo-finiteness}, as it certifies the minimal size of the region that the responder must preserve.

\begin{remark*}
    The proof of \cref{thm:pseudo-finiteness} provides two lower bounds on the convergence rate of \(\OPT(T,\epsilon,\delta)\): exponential in \(T\) for \(\epsilon \neq 0\) and double-exponential for \(\epsilon = 0\). On the upper-bound side, obtaining a matching rate for \(\delta >0\) appears nontrivial, and the optimality of the known lower bounds remains unclear.
    
    In the purely multiplicative case \(\delta=0\), however, \(\OPT_X(\epsilon,0)=0\), and a matching exponential upper bound follows from a standard grid-refinement argument: the reconstructor queries a uniform grid of fixed size (depending only on the dimension), selects the grid point with the smallest reported distance, then recenters a new fixed-size grid at that point and rescales to a smaller neighborhood. Iterating this geometrically shrinks the feasible region, yielding an exponential upper bound on \(\OPT_X(T,\epsilon,0)\).
\end{remark*}

%:the reconstructor can query a uniform grid of fixed size (depending only on the dimension), identify the grid point with the smallest reported distance, and then refine the search by querying a new fixed-size grid centered around that point and scaled to a smaller neighborhood. Repeating this process geometrically shrinks the feasible region, yielding an exponential upper bound on \(\OPT_X(T, \epsilon, \delta)\).

\paragraph{Organization.}
In the next section (Section~\ref{sec:examples}), we analyze and discuss basic examples of the reconstruction game. In Section~\ref{sec:proof-overview}, we provide a high-level overview of the main technical ideas used in our proofs. The related work is deferred to Section~\ref{sec:related-work}, which surveys relevant literature from learning theory, privacy, and geometry.  The complete formal proofs are presented in Sections~\ref{app:notation} through~\ref{app:pseudo}. Appendix~\ref{appendix} collects technical lemmas from geometry and topology and provides full proofs of the examples sketched in \cref{sec:examples}, along with additional examples that further clarify the game.

% The appendices in the full version of the paper contain the full formal development: Appendix~\ref{app:notation} summarizes general notation and basic tools used throughout the paper; Appendices~\ref{app:feasible}, \ref{app:geom}, \ref{app:Jung}, and \ref{app:pseudo} contain the detailed proofs of our main theorems.

\section{Examples}\label{sec:examples}
This section presents illustrative examples of the reconstruction game in a variety of metric spaces. These examples shed light on different aspects of the problem, including the necessity of the assumptions in our main theorems and the range of geometric behaviors that can arise. They also help clarify the role of total boundedness in Theorem~\ref{thm:Jung}, and lead naturally to an open question about pseudo-finite totally bounded spaces. In contrast to the following sections—which focus more heavily on Euclidean metric spaces in the context of Theorem~\ref{thm:pseudo-finiteness}—this section is technically lighter and features some more ``exotic'' spaces. Full proofs of the examples discussed here appear in Appendix~\ref{app:examples}.

\subsection{Total Boundedness in Theorem~\ref{thm:Jung}}
\label{sec:not_tot_bound}
The first main result (Theorem~\ref{thm:Jung}) characterizes the limiting reconstruction error in terms of the diameter-radius profile \( \mathtt{e}_X(\cdot) \), assuming that the metric space \( X \) is totally bounded. The following examples illustrate that this assumption is essential: if total boundedness is lifted, even seemingly natural spaces allow the responder to prevent the reconstructor from obtaining any meaningful approximation—specifically, an error bounded away from zero, or even infinite.

\begin{example}[Unbounded Space: The Real Line]
\label{ex:unbounded-real}
We begin with a simple case: \(\mathbb{R}\) with its standard Euclidean metric. This space is not bounded (and hence not totally bounded), and the responder can exploit its unboundedness to maintain extremely large feasible regions throughout the game. 
A formal proof is given in Appendix~\ref{app:examples}.
% For simplicity, we consider the one-dimensional case.

% Let \( X = \mathbb{R} \), and fix any \( \epsilon > 0 \). We claim that for any finite number of queries \( T \), the optimal reconstruction error satisfies \( \OPT_X(T, \epsilon, 0) = \infty \). Indeed, without loss of generality, assume \( \delta = 0 \); this assumption only further restricts the responder. Suppose the reconstructor queries a point \( q_1 \in \mathbb{R} \). The responder replies with an arbitrarily large value \( {\hat d}_1 = r \), hence \(
% \dist(q_1, x^\star) \in \left[ \frac{r}{1+\epsilon},\; r(1+\epsilon) \right].\)
% This means that the feasible region contains two intervals, each of length 
% \[
% R := r\left( (1+\epsilon) - \frac{1}{1+\epsilon} \right)=\Theta(r\cdot\epsilon),
% \]
% centered around $q_1 \pm r$. In subsequent rounds, the reconstructor can try to shrink the region through additional queries, but as we will show in \textcolor{black}{add reference}, even after \( T \) rounds, the responder can maintain a feasible region of length roughly \( R \cdot \exp(-O(T)) \). Since \( R \) can be made arbitrarily large by choosing \( r \) large, the responder can force the reconstruction error to be unbounded.
\end{example}
The previous example showed that in some unbounded metric spaces, such as \( \mathbb{R} \), the responder can force the reconstruction error to be arbitrarily large. This naturally raises the question: could boundedness alone suffice for the conclusion of Theorem~\ref{thm:Jung}? That is, can we strengthen the theorem by replacing total boundedness with the weaker assumption of boundedness?
The answer is negative:
\begin{example}[Bounded but Not Totally Bounded: Discrete Countable Space]
\label{ex:bounded-discrete}
 Consider the space \( X = \mathbb{N} \), the set of natural numbers equipped with the discrete metric: \( \dist(i,j) = 0 \) if \( i = j \), and \( \dist(i,j) = 1 \) otherwise. This space is bounded (diameter 1) but not totally bounded. 
 
 Now, note that even if the responder must be fully honest (i.e., \( \epsilon = \delta = 0 \)), it can always answer \( \hat{d}_t= 1\). This ensures that the feasible region after every round remains an infinite subset of \( X \) in which all points are pairwise at distance 1. Consequently, the responder can choose a consistent point of distance $1$ from the point guessed by the reconstructor, yielding an approximation error equal to the diameter of the space.
\end{example}

\subsection{Pseudo-Finiteness}
Although our main result about pseudo-finiteness focuses on convex Euclidean spaces, the phenomenon is more subtle in general metric spaces. In this section, we present three infinite metric spaces. Two of these spaces are \((\epsilon, \delta)\)-pseudo-finite for all values of \(\epsilon, \delta \ge 0\), while the third is not even \((0,0)\)-pseudo-finite. These examples highlight the diversity of possible behaviors in general metric spaces and motivate an open question concerning the structural nature of pseudo-finiteness in totally bounded spaces.

\begin{example}[Sparse Subsets of the Real Line]
\label{ex:sparse-reals}
Let \( X = \{0\} \cup \{2^{2^n} : n \in \mathbb{N}\} \subset \mathbb{R} \) with the standard Euclidean metric.
Then \( X \) is \((\epsilon,\delta)\)-pseudo-finite for every \( \epsilon, \delta \ge 0 \).

To see this, let the reconstructor begin by querying the point \( q_1 = 0 \). The response \( \hat{d}_1 \) yields a feasible region consisting of a finite subset of \( X \), whose size is bounded by a constant \( N(\epsilon,\delta) \) that depends only on the noise parameters (and not on the specific value of \( \hat{d}_1 \)). This is because the set~\( X \), when viewed as a monotone sequence, grows asymptotically faster than any geometric progression. After this initial step, the reconstructor continues to query all points in the feasible region to identify an optimal approximation.
\end{example}

The above example is unbounded. This raises the question of whether there exist bounded infinite spaces that are \((\epsilon, \delta)\)-pseudo-finite for all \(\epsilon, \delta\). The next example shows that the answer is yes.

\begin{example}[Countable Discrete Metric Space Revisited]
\label{ex:discrete-pseudofinite}
Recall the space \( X = \mathbb{N} \) with the discrete metric: \( \dist(x,y) = 0 \) if \( x = y \), and 1 otherwise. This space is bounded, with diameter 1. As previously discussed (see Example~\ref{ex:bounded-discrete}), we have \( \OPT_X(\epsilon,\delta) = 1 \) for all \( \epsilon, \delta \ge 0 \). Therefore, the reconstructor can achieve optimal performance without submitting any queries, simply by outputting any fixed point in the space. Thus, \( X \) is \((\epsilon,\delta)\)-pseudo-finite for all \( \epsilon, \delta \ge 0 \).
\end{example}

These two examples motivate the following open question: can similar behavior occur in totally bounded spaces?:

\begin{question}
Let \( X \) be a totally bounded metric space. Are the following two statements equivalent?
(i) \( X \) is finite. (ii) \( X \) is \((\epsilon,\delta)\)-pseudo-finite for all \( \epsilon, \delta \ge 0 \).
\end{question}

We conclude this section by presenting a totally bounded metric space that is not \((0,0)\)-pseudo-finite:

\begin{example}[Infinite binary strings]\label{ex:binary-jung}
Let \( X = \{0,1\}^{\mathbb{N}} \) be the space of infinite binary sequences, equipped with the standard ultrametric,\footnote{
This metric satisfies the \emph{ultrametric inequality}:
\(
d(x,z) \le \max\{d(x,y), d(y,z)\},
\)
which is stronger than the standard triangle inequality. It implies, for instance, that all triangles are isosceles with the two longer sides equal.
}
defined by \( d\left((\alpha_i)_{i\in \NN}, (\beta_i)_{i\in \NN}\right) = 2^{-j} \), where \( j \) is the first index at which \( \alpha_j \neq \beta_j \). Then \( X \) is a compact metric space that is {\bf not} \( (0,0) \)-pseudo-finite. The proof appears in \cref{ex:binary_sequences}.

\end{example}

\section{Technical Overview}
\label{sec:proof-overview}
In this section, we outline the key ideas behind the proofs of Theorem~\ref{thm:Jung} and Theorem~\ref{thm:pseudo-finiteness}; complete proofs are deferred to \cref{app:Jung,app:pseudo}. To keep the exposition focused on the central arguments, we omit technical complications arising from cases where suprema or infima are not attained. These can be handled with standard limiting arguments but would introduce additional notation and obscure the main ideas.

% Both proofs of Theorem~\ref{thm:Jung} and Theorem~\ref{thm:pseudo-finiteness} share a common core idea for deriving lower bounds, namely, a carefully designed responder strategy that maintains a feasible region of large diameter throughout the interaction. 

% The key difference lies in the use of extremal sets in a metric space—subsets of bounded diameter that attain the maximal Chebyshev radius—as a basis for constructing the required responder strategy. The proof of Theorem~\ref{thm:Jung} uses a single extremal set to construct a responder strategy that maintains the Chebyshev radius of the feasible region at least the optimal value. 

% In contrast, to obtain stronger statements in a narrower setting (namely, in bounded Euclidean spaces compared to totally bounded metric spaces), the proof of Theorem~\ref{thm:pseudo-finiteness} relies on the Minkowski sum of a specific extremal set with a ball of controlled diameter. The responder maintains this diameter uniformly throughout the process, which requires substantial analytical computation and makes the proof of Theorem~\ref{thm:pseudo-finiteness} considerably more computationally complex than that of Theorem~\ref{thm:Jung}.

% It is worth noting, however, that the proof of Theorem~\ref{thm:Jung} is challenging in its own right: because the statement of the theorem holds in such generality, the upper-bound argument relies on advanced mathematical machinery developed for general metric spaces.

\subsection{Proof of Theorem~\ref{thm:Jung}}
We begin by recalling the core assertion of Theorem~\ref{thm:Jung}. It characterizes the optimal reconstruction error \(\OPT_X(\epsilon,\delta)\) in terms of the geometry of the metric space and the noise parameters \(\epsilon\) and \(\delta\). Specifically, it asserts that \(\OPT_X(\epsilon,\delta)\) equals the maximum Chebyshev radius among all subsets of~\(X\) with diameter at most \((2+\epsilon)\delta\):
\[
\OPT_X(\epsilon,\delta) = \mathtt{e}_X\big((2+\epsilon)\delta\big).
\]

To prove this, we begin by analyzing an idealized setting in which the reconstructor is allowed to query all points in the space. Of course, this is unrealistic in infinite spaces—but it serves as a useful thought experiment for understanding the limits of reconstruction.

Each query-answer pair \((q, r)\) determines a \emph{feasible region} \(\Phi(\{q, r\})\), which consists of all points whose noisy distances to \( q \) are \((\epsilon, \delta)\)-indistinguishable from \( r \). The intersection of all these regions gives the overall feasible region of the interaction, denoted by \(\Phi := \Phi(\{q, r_q\}_{q \in X})\).

\paragraph{Upper Bound.}
In the idealized case where \emph{all} points in the space are queried, a simple yet insightful argument shows that the diameter of the feasible region \( \Phi \) is at most \( (2+\epsilon)\delta \). Indeed, for any two points \( A, B \in \Phi \), since \( B \) was queried and \( A \) remained feasible, the reported noisy distance must not exceed \(\delta\), and therefore the true distance \(\dist(A,B)\) cannot exceed \((2+\epsilon)\delta\).
By letting the reconstructor output the Chebyshev center of \( \Phi \), the reconstruction error is at most \( \mathtt{e}_X\left((2+\epsilon)\delta\right) \).

When only finitely many queries are allowed, however, the reconstruction error can be significantly larger than in the idealized case; as shown in Section~\ref{sec:not_tot_bound}, there exist spaces in which this discrepancy is arbitrarily large.

Nevertheless, if the metric space \( X \) is \emph{totally bounded}, the reconstructor can approximate the idealized strategy arbitrarily well: by querying all points in a sufficiently dense finite cover, one ensures that the feasible region has diameter arbitrarily close to \( (2+\epsilon)\delta \). 
Such a finite cover exists by definition: a metric space is totally bounded if, for every \( \alpha > 0 \), it admits a finite \( \alpha \)-cover—that is, a finite subset such that every point in the space lies within distance \( \alpha \) of some point in the cover. Denote by \( N_\alpha \) the number of points in an \( \alpha \)-cover of the metric space \( X \). As illustrated in Figure~\ref{fig:alpha-net}, after \( N_\alpha \) queries the reconstructor can guarantee that the diameter of the feasible region is less than \( (2+\epsilon)\delta + \alpha' \), where \(\alpha' = ((1+\epsilon)^2 + 1)\alpha\). Hence, by outputting the Chebyshev center of the feasible region, the reconstructor ensures that the worst-case error after \( N_\alpha \) queries does not exceed \( \mathtt e_X((2+\epsilon)\delta + \alpha') \). This follows from the definition of the diameter-radius profile \(\mathtt e_X(\beta)\), which is the maximum Chebyshev radius over all subsets of \( X \) with diameter at most \(\beta\).

One might be inclined to conclude that the analysis is complete, as the function \(\mathtt e_X\) appears to be continuous. However, this inference is generally incorrect. For arbitrary metric spaces, the diameter-radius profile \(\mathtt e_X\) is not necessarily continuous. For example, in finite metric spaces, the function \(\mathtt e_X\) is not continuous.

However, for totally bounded metric spaces, the function \(\mathtt e_X\) can be shown to be right-continuous, which is sufficient for establishing the desired upper bounds. The assumption of total boundedness is essential, as there exists a non–totally bounded metric space for which the corresponding function \(\mathtt e_X\) fails to be right-continuous (see \cref{ex:not_right_eX} in \cref{app:examples}). This observation indicates that proving the right-continuity of the diameter–radius profile is more delicate than it might initially appear.

We prove that the function \(\mathtt e_X\) is right-continuous for totally bounded metric spaces using the theory of \emph{hyperspaces}.
Namely, given a metric space \(X\), one considers the space of (nonempty) \emph{compact} subsets of \(X\), denoted \(\cK(X)\), equipped with metrics induced by the metric on \(X\). The most standard choice is the \emph{Hausdorff metric}: for subsets \(S_1,S_2\subseteq X\),
\[
  d_H(S_1,S_2)
  = \max\!\left\{
      \adjustlimits \sup_{x\in S_1}\inf_{y\in S_2} \dist(x,y),\;
      \adjustlimits  \sup_{y\in S_2}\inf_{x\in S_1} \dist(x,y)
    \right\}.
\]
A variety of classical results are known for \((\cK(X),d_H)\); for instance, when \(X\) is compact, the hyperspace \(\cK(X)\) is compact as well. This is a classical fact in metric topology; see, e.g.,~\cite{IllanesNadler1999}[Theorem 3.5].

Both the diameter and the Chebyshev radius of a set are continuous functions on \((\cK(X), d_H)\); this follows by bounding their variation in terms of the Hausdorff distance (see the detailed argument in \cref{app:right_continuity_e}). Together with the compactness of \((\cK(X), d_H)\), this yields, via a compactness argument, that \(\mathtt e_X\) is right-continuous for compact metric spaces. For a totally bounded metric space~\(X\), in turn, one can show that \(\mathtt e_X = \mathtt e_{\hat X}\), where \(\hat X\) denotes the completion of \(X\). Since the completion of a totally bounded metric space is compact by the classical Heine–Borel characterization for metric spaces (compact \(\Leftrightarrow\) complete and totally bounded), it follows that \(\mathtt e_X\) is right-continuous for totally bounded metric spaces as well. All techniques and formal proofs for the right-continuity of the diameter-radius profile~\(\mathtt e_X\) are presented in \cref{app:right_continuity_e}. The full proof of the upper bound appears in \cref{app:Jung-upper}.

%The net construction that yields the diameter bound appears in \cref{app:Jung-upper}. 

% However, once we move beyond familiar settings—finite graphs, bounded subsets of \(\RR^n\), and similar examples—to the much broader class of totally bounded metric spaces, it is no longer clear that permitting the diameter to exceed \((2+\epsilon)\delta\) by an arbitrarily small amount \(\alpha\) still restricts the Chebyshev radius to exceed \(\mathtt{e}_X\!\bigl((2+\epsilon)\delta\bigr)\) by only an equally small margin \(\beta_\alpha\).

% To justify this, we invoke tools that treat compact subsets themselves as points of a metric space.  
% Equipping the family of compact sets with the Hausdorff (``natural'') metric, we note that when the space~\(X\) is compact, this hyperspace is compact as well.  
% That observation lets us prove that the function \(\mathtt{e}_X\) is right–continuous for compact metric spaces. Using this fact, we extend the result to totally bounded metric spaces.

% The right-continuity, in turn, yields the desired estimate. The topological arguments used to establish the right-continuity of \(\mathtt e_X\) are given in \cref{app:right_continuity_e}.

\paragraph{Lower Bound.}
The crucial observation is that at the beginning of the game, the responder may select \emph{any} subset \( S \subset X \) of diameter at most \( (2+\epsilon)\delta \), and maintain the invariant \( S \subseteq \Phi \) throughout the interaction: In response to each query \( q \), the responder identifies a point \( S_{\min} \in S \) that minimizes the distance to \( q \), and returns the perturbed value  
\[
r := (1+\epsilon) \cdot \dist(q, S_{\min}) + \delta.
\]  
A simple calculation, which relies only on the triangle inequality, shows that
every \( s \in S \) satisfies \( \dist(q,s) =_{\epsilon,\delta} r \), and hence \( S \) remains feasible.

After the interaction concludes, given the reconstructor's final guess, the responder can choose a secret point at distance no less than \( r(S) \) inside \( S\subseteq \Phi \). This ensures that no reconstructor can guarantee an error smaller than \(\mathtt{e}_X\left((2+\epsilon)\delta\right)\). 
A precise description of the responder strategy that preserves an extremal set is presented in \cref{app:Jung-lower}.

\begin{figure}[ht]
  \centering
  \begin{minipage}[t]{0.48\textwidth}
    \centering
    \begin{tikzpicture}[scale=1.5]
      % Feasible region Φ (ellipse)
      \fill[blue!20,opacity=.5] (0,0) ellipse (1 and 0.6);
      \draw[blue!60,thick] (0,0) ellipse (1 and 0.6)
            node[below right,font=\small] {$\Phi$};

      % Minimal enclosing ball (dashed circle)
      \draw[dashed,black] (0,0) circle (1);

      % Chebyshev centre
      \filldraw[red] (0,0) circle (0.045)
            node[ right,font=\scriptsize] {Chebyshev center};

      % Two boundary points q and x (not symmetric)
      \coordinate (q) at (55:1 and 0.6);   % boundary point q
      \coordinate (x) at (170:1 and 0.6);  % boundary point x (off-centre)

      \fill (q) circle (0.03) node[above,font=\scriptsize] {$q$  - query};
      \fill (x) circle (0.03) node[left,font=\scriptsize] {$x$};

      % Segment q--x with distance label
      \draw[<->,gray] (q) -- (x)
            node[midway,above=0pt,font=\scriptsize] {$\displaystyle \le (2+\epsilon)\delta$};
    \end{tikzpicture}
    \caption{Feasible region \(\Phi\) (blue) of the idealized case}
    \label{fig:feasible-region-idealized}
  \end{minipage}\hfill
  \begin{minipage}[t]{0.48\textwidth}
    \centering
    \begin{tikzpicture}[scale=1.5, every node/.style={font=\scriptsize}]
      \def\beta{0.35}

      % Feasible region Φ
      \fill[blue!20,opacity=.5] (0,0) ellipse (1 and 0.6);
      \draw[blue!60,thick] (0,0) ellipse (1 and 0.6)
            node[below right,font=\small] {$\Phi$};

      % α-NET
      \coordinate (C) at (0,0);
      \fill[green!70!black] (C) circle (0.02);
      \fill[green!25,opacity=.3] (C) circle (\beta);
      \draw[green!60] (C) circle (\beta);

      \foreach \ang in {0,45,90,135,180,225,270,315}{
        \coordinate (P\ang) at (\ang:{0.7} and {0.42});
        \fill[green!70!black] (P\ang) circle (0.02);
        \fill[green!25,opacity=.3] (P\ang) circle (\beta);
        \draw[green!60] (P\ang) circle (\beta);
      }

      \coordinate (y) at (60:1 and 0.6);
      \coordinate (x) at (170:1 and 0.6);
      \fill (y) circle (0.03) node[right] {$y$};
      \fill (x) circle (0.03) node[left]  {$x$};

      \coordinate (q) at (P45);
      \fill[red] (q) circle (0.03) node[below right] {$q$ - query};

      \draw[dashed,gray] (q) -- (y) node[midway,right=1pt] {$<\alpha$};
      \draw[gray,thick] (y) -- (x) node[midway,above=1pt] {$\le (2+\epsilon)\delta+\alpha'$};
      \draw[gray] (q) -- (x);

      \begin{scope}[shift={(1.45,0.7)}]
        \fill[green!25,opacity=.3] (0,0) circle (0.07);
        \draw[green!60] (0,0) circle (0.07);
        \node[right=0.1] at (0,0) {$\alpha$-ball};
        \fill[green!70!black] (0,-0.25) circle (0.02);
        \node[right=0.1] at (0,-0.25) {net point};
        \node[right=0.1] at (0,-0.5) {$\alpha' = ((1+\epsilon)^2+1)\alpha$};
      \end{scope}
    \end{tikzpicture}
    \caption{Feasible region \(\Phi\) (blue) of the finite interaction}
    \label{fig:alpha-net}
  \end{minipage}
\end{figure}

\subsection{Proof of Theorem~\ref{thm:pseudo-finiteness}}

We now turn to the proof of Theorem~\ref{thm:pseudo-finiteness}, which establishes a dichotomy for pseudo-finiteness in bounded convex subsets of Euclidean space. Specifically, the theorem states that a bounded convex set \( X \subset \mathbb{R}^n \) is \( (\epsilon, \delta) \)-pseudo-finite if and only if \( \dim(X) = 1 \) and \( \epsilon = 0 \). In all other cases—namely, when \( \dim(X) > 1 \) or \( \epsilon > 0 \)—the reconstruction error cannot reach its optimal value in finitely many steps.

A natural approach for proving non-pseudo-finiteness is to design a responder strategy that gradually shrinks the feasible region, for example, by adding uniform random noise to the true distance in each response. However, this approach does not guarantee the desired outcome: specifically, it fails to ensure that the reconstruction error remains strictly greater than the optimum \( \OPT(\epsilon, \delta) \) at all finite $T$. In fact, such strategies may result in convergence to a strictly smaller value, making them suboptimal for the responder.

To overcome this, our proof explicitly constructs a responder strategy that, at each round, ensures the feasible region contains a subset guaranteeing that the reconstruction error remains strictly larger than the optimum \( \OPT(\epsilon, \delta) + \alpha_T \), where \(\alpha_T > 0\) depends only on the number of rounds.

% Specifically, the feasible region is required to contain a neighborhood around an \emph{extremal set}---a subset \( S \subseteq X \) of diameter \( (2+\epsilon)\delta \) such that
% \[
% r(S) = \mathtt{e}_X\big((2+\epsilon)\delta\big).
% \]

% In Euclidean space, extremal sets are exactly those that contain the vertices of a regular simplex.\footnote{This fact is not directly relevant to our proof but follows straightforwardly from Jung's theorem; see~\cite{Blumenthal1953}.}

% Specifically, the feasible region is required to contain an \(\alpha\)-neighborhood of the vertices \(\{x_i\}_{i=0}^n\) of a regular simplex \(\Delta\) with edge length \((2+\epsilon)\delta\); that is, the Minkowski sum of the set \(\{x_i\}_{i=0}^n\)\! and the ball \(B(0,\alpha)\):
% \[
%   \bigcup_{i=0}^n B(x_i, \alpha).
% \]
% Once the feasible region contains such a set, the responder guarantees that the reconstruction error is at least \(\OPT(\epsilon, \delta) + \alpha\), regardless of the reconstructor’s guess.
% ---a subset \( S \subseteq X \) of diameter \( (2+\epsilon)\delta \) such that
% \[
% r(S) = \mathtt{e}_X\big((2+\epsilon)\delta\big).
% \]

This approach parallels the lower-bound strategy employed in Theorem~\ref{thm:Jung}. There, the responder preserved an extremal set to ensure the Chebyshev radius did not fall below \( \OPT(\epsilon, \delta) \). However, to establish non-pseudo-finiteness, it is insufficient to preserve a region whose radius merely equals the optimum. Instead, it is necessary to ensure that the Chebyshev radius of the feasible region remains strictly greater than the limiting value for all finite \( T \).

Our strategy preserves an \( \alpha \)-neighborhood of the vertices \(\{x_i\}_{i=0}^n\) of a regular simplex \( \Delta \)\footnote{Throughout this work, by a “simplex” we usually mean the set of its vertices, that is, \( n+1 \) affinely independent points in the Euclidean space \(X\) of dimension~\(n\).}. We denote this neighborhood as~\( \Delta_{\alpha} \), where \(\alpha > 0\) depends only on the number of rounds~\(T\). Formally, the \(\alpha\)-neighborhood of the simplex \(\Delta\) with vertices \(\{x_i\}_{i=0}^n\) is defined as
\begin{equation}\label{eq:neighborhood_simplex}
  \Delta_{\alpha} := \cup_{i=0}^n B(x_i, \alpha),
  \qquad
  B(x_i, \alpha) \text{ denotes the Euclidean ball of radius } \alpha \text{ centered at } x_i.
\end{equation}

This strategy ensures the feasible region contains a regular simplex of diameter \( (2+\epsilon)\delta + \sqrt{\tfrac{2(n+1)}{n}} \alpha \). This implies its Chebyshev radius is at least \( \OPT(\epsilon, \delta) + \alpha \).

Our strategy proceeds as follows.  Assume that at round~\(t\) the feasible
region already contains an \(\alpha_t\)-neighborhood of a regular simplex
\(\Delta\).  Upon receiving the next query \(q_t\), we pick a radius
\(\alpha_{t+1}\) determined solely by \(t+1,\epsilon,\delta\); then we reply
with an appropriate noisy distance \(r_t\) and, if necessary, replace
\(\Delta\) by a new extremal simplex \(\Delta'\) so that
\[
  (\Delta')_{\alpha_{t+1}} \subset (\Delta)_{\alpha_t}
  \quad\text{and}\quad
  (\Delta')_{\alpha_{t+1}} \subset \text{(updated feasible region)}.
\]
This step is then repeated indefinitely, maintaining the
Chebyshev radius strictly above \(\OPT\). As a result, \(X\) is not
pseudo-finite.

The primary challenge is to establish a uniform lower bound on \( \alpha_t \) that depends only on the round~\( t \), and not on the specific query \( q_t \).
We note in passing that it is relatively easy to give a bound on \( \alpha_{t+1} \) that depends on both \( t \) and the query \( q_t \); however, such a bound is insufficient for our purposes, as it does not yield a general lower bound on \( \OPT_T \) valid for all reconstructor strategies, which is essential for ruling out pseudo-finiteness.
% does not guarantee that the reconstructor cannot approach the secret point arbitrarily closely, having already decided how close they wish to get in advance.

Establishing a uniform bound requires careful analysis of each query type. For some queries, it is straightforward to obtain a sufficiently large neighborhood of a simplex \(\Delta'\) contained within the previous neighborhood, while others require a more detailed geometric argument.

% \paragraph{Determining the Maximal Surviving Neighborhood.}

% To address these challenges, we develop a tool — fully detailed in \cref{app:feasible} — that determines whether there exists a response that keeps a neighborhood of a given simplex \( \Delta \) within the feasible region. 

% Let \( \Delta \) be a regular simplex and let \( \alpha > 0 \). Suppose the reconstructor submits a query point \( q \). 
\paragraph{Determining the Maximal Surviving Neighborhood.}

To address these challenges, we consider the following question: under what conditions does there exist an answer that the responder can give to the query \(q\) such that the \(\alpha\)-neighborhood (see \cref{eq:neighborhood_simplex}) of \(\Delta\) remains entirely within the feasible region?

% The answer is as follows: there exists such a response if and only if \( r^{\min}_q(\Delta_\alpha)\le r^{\max}_q(\Delta_\alpha) \), where
% \[
% r^{\min}_q(\Delta_\alpha) := \frac{\max_{s \in \Delta_\alpha} \dist(s, q) - \delta}{1+\epsilon}, 
% \qquad
% r^{\max}_q(\Delta_\alpha) := (1+\epsilon) \min_{s \in \Delta_\alpha} \dist(s, q) + \delta.
% \]
% In particular, there exists a response \( r \) that preserves the feasibility of \( \Delta_\alpha \) if and only if \( r^{\min}_q(\Delta_\alpha) \le r^{\max}_q(\Delta_\alpha) \). 
% Here, \( r^{\min}_q\) places the farthest point of \( \Delta_\alpha \) on the outer boundary of the feasible region, and \( r^{\max}_q\) places the nearest point on the inner boundary (see Figure~\ref{fig:separate-feasible-sets}). 

Such a response exists if and only if \( r^{\min}_q(\Delta_\alpha)\le r^{\max}_q(\Delta_\alpha) \), where \( r^{\min}_q\) places the farthest point of \( \Delta_\alpha \) on the outer boundary of the feasible region, and \( r^{\max}_q\) places the nearest point on the inner boundary (see \cref{fig:feasible-sets-minimal}, \cref{fig:feasible-sets-maximal}).

\begin{figure}[ht]
  \centering
%--------------------------- LEFT  : blue annulus ----------------------------
  \begin{minipage}[t]{0.48\textwidth}
    \centering
    \begin{tikzpicture}[scale=0.7]

      % radii ----------------------------------------------------------------
      \def\rBlueInner{0.40}
      \def\rMax      {1.90}

      % blue annulus ---------------------------------------------------------
      \fill[blue!45, opacity=.30] (0,0) circle (\rMax);
      \fill[white]                (0,0) circle (\rBlueInner);

      % query point ----------------------------------------------------------
      \filldraw[black] (0,0) circle (0.045) node[below right=2pt] {$q$};

      % vertices of regular Δ ------------------------------------------------
      \coordinate (P1) at  ( 0.6894 , -0.1216 );   % S_min  (|qP1| = 0.70)
      \coordinate (P2) at  (-0.7448 ,  0.9674 );   % interior vertex (|qP2|=1.22)
      \coordinate (P3) at  ( 0.9153 ,  1.6650 );   % S_max  (|qP3| = 1.90)

      \draw[dashed, thick] (P1)--(P2)--(P3)--cycle;

      \foreach \pt/\lbl in {P1/$S_{\min}$, P2/, P3/$S_{\max}$}
        \filldraw[red] (\pt) circle (0.045) node[above right=2pt] {\lbl};

      \node at (0.2866,0.8369) {$\Delta$};
    \end{tikzpicture}
    \caption{\(\Phi(q,r^{\min})\) (blue)}
        \label{fig:feasible-sets-minimal}
  \end{minipage}
  \hfill
%--------------------------- RIGHT : orange annulus --------------------------
  \begin{minipage}[t]{0.48\textwidth}
    \centering
    \begin{tikzpicture}[scale=0.7]

      % radii ----------------------------------------------------------------
      \def\rMin        {0.70}
      \def\rOrangeOuter{2.30}

      % orange annulus -------------------------------------------------------
      \fill[orange!45, opacity=.30] (0,0) circle (\rOrangeOuter);
      \fill[white]                  (0,0) circle (\rMin);

      % query point ----------------------------------------------------------
      \filldraw[black] (0,0) circle (0.045) node[below right=2pt] {$q$};

      % same Δ ---------------------------------------------------------------
      \coordinate (P1) at  ( 0.6894 , -0.1216 );
      \coordinate (P2) at  (-0.7448 ,  0.9674 );
      \coordinate (P3) at  ( 0.9153 ,  1.6650 );

      \draw[dashed, thick] (P1)--(P2)--(P3)--cycle;

      \foreach \pt/\lbl in {P1/$S_{\min}$, P2/, P3/$S_{\max}$}
        \filldraw[red] (\pt) circle (0.045) node[above right=2pt] {\lbl};

      \node at (0.2866,0.8369) {$\Delta$};
    \end{tikzpicture}
    \caption{\(\Phi(q,r^{\max})\) (orange)}
    \label{fig:feasible-sets-maximal}
  \end{minipage}
  % \caption{Left: feasible region for the maximal answer \(r^{\min}_q(\Delta)\).  
  %          Right: feasible region for the minimal answer \(r^{\max}_q(\Delta)\).}
\end{figure}

The larger the radius \( \alpha \) of the neighborhood \( \Delta_\alpha \), the smaller the gap \( r^{\max} - r^{\min} \) becomes. 
Solving the equation \( r^{\max}(\Delta_\alpha) - r^{\min}(\Delta_\alpha) = 0 \) for \( \alpha \) yields the exact value \(\alpha^\star(\Delta, q)\) of the largest surviving neighborhood after querying \( q \). The quantity \( \alpha^\star(*, q) \) can be viewed as a function on the space of regular simplexes.
The derivation of the exact formula for \( \alpha^\star \) is presented in \cref{app:feasible}.

% Let \( d_{\max} \) and \( d_{\min} \) denote the farthest and nearest distances from the query point \( q \) to the simplex \( \Delta \), respectively.
% A calculation (see \cref{lem:biggest_neighborhood} in 
% \cref{app:feasible}) shows that  \( \alpha^\star(\Delta, q) \) admits a natural decomposition \(
% \alpha^\star(\Delta, q) = \alpha_1 + \alpha_2 \)
% where 
% \[
% \alpha_1 = C^{(1)}_{\epsilon, \delta} \cdot \left( \diam \; \Delta - (d_{\max} - d_{\min})\right),
% \qquad
% \alpha_2 = C_\epsilon^{(2)} \cdot d_{\min}.
% \]
% for \(C^{(1)}_{\epsilon, \delta} >0, C_\epsilon^{(2)} \ge 0\) depending only on \(\epsilon, \delta\) and dimension \(n\). 

% Both terms are nonnegative: the term \( \alpha_1 \) decreases as the gap \( d_{\max} - d_{\min} \) increases, while \( \alpha_2 \) grows \emph{linearly} with the distance of \( q \) from the simplex \( \Delta \), \emph{unless} \( \epsilon = 0 \). In this case, \( \alpha_2 = 0 \) as well. 

% This decomposition is important because it highlights the distinction between additive-only and multiplicative noise.

%In the purely additive case (\( \epsilon = 0 \)), the offset term vanishes: \( \alpha_2 = 0 \), which makes the task significantly more challenging.

\paragraph{Additive–Only vs.\ Multiplicative Noise.}
Both responder strategies, additive-only and mixed-noise, are based on the same principle: for a fixed simplex \( \Delta \) and a target neighborhood radius \( \alpha_{t+1} \), along with the previous round's neighborhood radius \( \alpha_t > \alpha_{t+1} \), we partition the space into \((\Delta, \alpha_{t+1})\)-\emph{good} and \emph{bad} regions. A query point \( q \) is considered \emph{good} if there exists a response that preserves a neighborhood of radius at least \( \alpha_{t+1} \) of \(\Delta\) within the feasible region; equivalently, if the maximal surviving neighborhood satisfies \( \alpha^\star(\Delta, q) \ge \alpha_{t+1} \). Otherwise, \( q \) is \emph{bad}.

When \( q \) is good, the responder can maintain the \(\alpha_{t+1}\)-neighborhood of the current simplex \( \Delta \) within the feasible region. The critical distinction between regimes occurs when \( q \) is bad. In this case, it is necessary to find another regular simplex \( \Delta' \) in the \(\alpha_t\)-neighborhood of \(\Delta\) such that the point \( q \) is now \((\Delta', \alpha_{t+1})\)-good. In the multiplicative case (\( \epsilon > 0 \)), the responder can \emph{translate} the simplex \( \Delta \) slightly away from the query point \( q \) to ensure that \( \alpha^\star(\Delta', q) \ge \alpha_{t+1} \) and that \( \Delta'_{\alpha_{t+1}} \subset \Delta_{\alpha_t} \).

In contrast, when \( \epsilon = 0 \), translations of \( \Delta \) within its \(\alpha_t\)-neighborhood do not substantially change~\( \alpha^\star(*, q) \). 
To achieve a significant increase in \( \alpha^\star(*, q) \) in this case,  we \emph{rotate} the simplex~\( \Delta \). 
To achieve this successfully, the rotated simplex must preserve the identities of the closest and farthest points to the query, which requires a careful geometric analysis. 
In \cref{app:geom}, we develop the tools necessary to carry out this strategy. 
In dimensions \( n \ge 2 \), rotations allow us to maintain a surviving neighborhood indefinitely. 
In one dimension, where nontrivial rotations are not possible, this strategy fails for a good reason: one-dimensional intervals are pseudo-finite.

% This increases the term \( \alpha_2 \) in the decomposition of \( \alpha^\star(*, q) \), ensuring that \( \alpha^\star(\Delta', q) \ge \alpha' \) and that \( \Delta'_{\alpha'} \subset \Delta_\alpha \). 

% In contrast, when \( \epsilon = 0 \), since \( \alpha^\star(*, q) = \alpha_1 \), the responder is forced to decrease \( d_{\max} - d_{\min} \) in order to ensure that now \( \alpha^\star(\Delta', q) \ge \alpha' \). To do this successfully, we \emph{rotate} the simplex \( \Delta \) within its neighborhood. In \cref{app:geom} we develop the geometric tools necessary to carry out this strategy. In dimensions \( n \geq 2 \), rotations allow us to maintain a surviving neighborhood indefinitely. In one dimension, however—where nontrivial rotations are not possible—this strategy fails for a good reason: one-dimensional intervals are pseudo-finite.

\section*{Acknowledgments and Disclosure of Funding}
We thank Nikita Gladkov for insightful discussions related to the problems studied in this work.

Shay Moran is a Robert J.\ Shillman Fellow; he acknowledges support by ISF grant 1225/20, by BSF grant 2018385, by Israel PBC-VATAT, by the Technion Center for Machine Learning and Intelligent Systems (MLIS), and by the the European Union (ERC, GENERALIZATION, 101039692). Views and opinions expressed are however those of the author(s) only and do not necessarily reflect those of the European Union or the European Research Council Executive Agency. Neither the European Union nor the granting authority can be held responsible for them.

\section{Related Work}
\label{sec:related-work}

We organize the discussion into two parts: research that focuses on the \emph{responder's perspective}, and research that centers on the \emph{reconstructor's perspective}. In both cases, the relevant literature is vast, so we focus on works most closely related to the questions studied in this paper.

\paragraph{The Responder's Perspective.}
The reconstruction game is closely related to problems studied in \emph{privacy-preserving data analysis}, where the goal is to answer queries on a sensitive dataset while limiting what an adversary can infer~\citep{DMNS06}.  
The foundational work of \citet{DinurN03} initiated this line of research by showing that approximate answers to too many counting queries enable the reconstruction of a large fraction of the database. Their model uses counting queries on binary datasets, which are essentially equivalent to Hamming distance queries on the Boolean cube~\( \{\pm 1\}^n \). This connection is illustrated in \cref{ex:DinurNissim_equivalence}.

Subsequent works have sharpened and generalized this reconstruction viewpoint. Notably, \citet{DMT07}, \citet{DY08}, and \citet{HaitnerMST22} provided refined attacks and bounds under weaker assumptions. More recently, \citet{balle2022reconstructing} and \citet{cummings2024attaxonomy} proposed formal definitions of \emph{reconstruction robustness} that relate privacy guarantees to the attacker's ability to reconstruct sensitive data.  
Recent work by \citet{CohenKMMNST25} further explores the foundations of reconstruction attacks, proposing a new definitional framework—Narcissus Resiliency—and uncovering connections to Kolmogorov complexity and classical notions such as differential privacy.

Surveys such as \citet{Survey2017} provide a comprehensive overview of privacy attacks and defenses, including reconstruction.
We also note the classical work of \citet{Erdos1963}, which (in disguise) studies a version of the reconstruction problem on the Hamming cube in the noiseless setting.

\paragraph{The Reconstructor's Perspective.}
Our work primarily studies the problem from the perspective of the \emph{reconstructor}, who seeks to locate a hidden point using approximate distance queries. Related problems have been studied under several guises. A classic formulation is the \emph{metric dimension} of a graph~\citep{harary1975metric,Slater75,MetricDimensionSurvey2023}, which asks for the smallest set of vertices such that all other vertices are uniquely identified by their distances to this set. This corresponds to an \emph{oblivious} version of the reconstruction game, where the reconstructor must submit all queries in advance.

A more sequential variant, closer to our setting, is the \emph{sequential metric dimension}~\citep{seager2013sequential,bensmail2020sequential,SeqMetricDimRand21}, which measures the number of adaptive queries needed to identify an unknown point. These works mostly consider noiseless settings on finite graphs. In contrast, our work allows noisy responses, considers general metric spaces, and studies the rate of convergence as a function of the number of queries.

The general formulation of locating a hidden point via distance queries has also appeared in applied contexts. For instance, the problem of reconstructing a physical quantity from noisy measurements arises in \emph{remote sensing}, including terrain mapping and atmospheric profiling. Classic references include \citet{Twomey1977} and \citet{Rodgers2000}, which formulate and analyze such problems as inverse problems under uncertainty. While much of this literature is algorithmic or statistical, our work offers a geometric and learning-theoretic view that complements these perspectives.

%-----------------------------------------------------------
\section{General notation and basic facts}
\label{app:notation}
%-----------------------------------------------------------
Let us remind the setup of the game and important concepts used throughout the proofs.

We work in a metric space \((X, \dist)\). The interaction lasts for a fixed number of rounds \(T\), labeled \(t = 1, 2, \dots, T\). In round \(t\), the reconstructor selects a query point~\(q_t \in X\). The responder then returns a real number \(r_t\) that represents a noisy distance from \(q_t\) to an as-yet-unspecified target, with multiplicative parameter \(\epsilon \ge 0\) and additive parameter \(\delta \ge 0\).

Formally, the reply must satisfy
\begin{align*}
    &\dist(x, q_i) \le (1+\epsilon)r_i + \delta,\\
    &r_i \le (1+\epsilon)\dist(x, q_i) + \delta
  \end{align*}
to at least one point $x\in X$. In other words, after each answer the set
\begin{equation*}
\Phi(\{q_i, r_i\}_{i=1}^{t}) := \left\{ x \in X \mid \text{ for all } 1 \le i \le t \colon \,\,  \begin{aligned}
    &\dist(x, q_i) \le (1+\epsilon)r_i + \delta,\\
    &r_i \le (1+\epsilon)\dist(x, q_i) + \delta
  \end{aligned}\right\}
\end{equation*}
is guaranteed to be non-empty. This set for the round $T $ is called the \emph{feasible region}.
In the end of the game the reconstructor outputs a guess point $\hat x_T$, and then the responder commits to a target point, choosing any \(x^\star \in \Phi_T\). 

The aim of the reconstructor is to minimize the distance $\dist(x^\star, \hat x_T)$, and of the responder is to maximize it. 

Let us denote the set of all reconstructors that play game for $T$ rounds by $\mathrm{RC}_T$, and the set of responders by $\mathrm{RSP}_T$. The final guess of the reconstructor $\cR \in \mathrm{RC}_T$ we will denote by $\hat x^\cR$ and the output secret point of the responder $\cA \in \mathrm{RSP}_T$ by $x^\star_\cA$. As recalled in Equation \ref{eq:optimal} the optimal error must be
\[
  \OPT_X(T,\epsilon,\delta) := \adjustlimits \inf_{\mathrm{RC}} \sup_{\mathrm{RSP}} \sup_{x \in \Phi_T} \dist(\hat{x}_T, x).
\]

\begin{claim}[Monotone error in $T$]\label{cor:monotone}
The function $T\mapsto\OPT_X(T,\epsilon,\delta)$ is non-increasing.    
\end{claim}
\begin{proof}[Proof of \Cref{cor:monotone}]
{We show that for every \( T \ge 0 \), the function \( T \mapsto \OPT_X(T,\epsilon,\delta) \) is non-increasing; that is,
\[
\OPT_X(T+1,\epsilon,\delta) \le \OPT_X(T,\epsilon,\delta).
\]

Fix any \( \alpha > 0 \). By the definition of \( \OPT_X(T,\epsilon,\delta) \), there exists a reconstructor \( \mathcal{R} \) that, after \( T \) queries, guarantees an error at most \( \OPT_X(T,\epsilon,\delta) + \alpha \) against any responder.

Now consider a new reconstructor \( \mathcal{R}' \) for \( T+1 \) rounds, which simulates \( \mathcal{R} \) for the first \( T \) queries, and then issues an arbitrary “dummy” query at round \( T+1 \), ignores the response, and simply outputs the same guess \( \hat{x}_{T+1} := \hat{x}_T \) that \( \mathcal{R} \) would have produced after \( T \) rounds.

Since the feasible region after \( T+1 \) queries is always contained in the feasible region after \( T \) queries, and since the final guess remains the same, the reconstruction error of \( \mathcal{R}' \) is at most \( \OPT_X(T,\epsilon,\delta) + \alpha \) for any responder.

As this holds for every \( \alpha > 0 \), it follows that
\[
\OPT_X(T+1,\epsilon,\delta) \le \OPT_X(T,\epsilon,\delta),
\]
as required.}
\end{proof}

\section{Proof of Theorem~\ref{thm:Jung}}\label{app:Jung}
%-----------------------------------------------------------

This section presents the complete proof of Theorem~\ref{thm:Jung}. The argument relies on geometric concepts, including the Chebyshev radius, the diameter of a set, and an additional metric space invariant referred to as the diameter-radius profile.

We begin by recalling the relevant definitions. For a subset \(S \subseteq X\), the \emph{Chebyshev radius} of \(S\), denoted~\(r(S)\), and the \emph{diameter} of~\(S\), denoted~\(\diam \; S\), are defined by
\[
  r(S) = \adjustlimits\inf_{q \in X} \sup_{x \in S} \dist(x, q), 
  \qquad 
  \diam \; S = \sup_{x, y \in S} \dist(x, y).
\]

The supremum and infimum of a set of real numbers \(\cA \subset \RR\) serve the same purpose as the maximum and minimum. The key difference is that the supremum or infimum may not be attained by any element of \(\cA\). In such cases, one can approximate it by a sequence \(\{a_i\}_{i \in \mathbb{N}} \subseteq \cA\) satisfying
\[
\lim_{i \to \infty} a_i = \sup \cA
\quad \text{or} \quad
\lim_{i \to \infty} a_i = \inf \cA.
\]

Another important invariant, the diameter-radius profile, represents the maximal radius of an enclosing ball over all subsets of \(X\) with diameter at most \(\alpha\). This invariant is formally defined as
\[
  \mathtt{e}_X(\alpha) := \sup_{\substack{S \subseteq X \\ \diam(S) \le \alpha}} r(S).
\]

In some cases, the supremum in the definition of the diameter-radius profile is attained; when that happens, we refer to the corresponding subset \(S \subset X\) as \emph{extremal}. In general, even when the supremum is not attained, we may consider a sequence of subsets \(\{S_m\}_{m \in \NN}\) of bounded diameter, \(\diam \; S_m \le \alpha\), whose Chebyshev radii converge to the supremum:
\[
  \lim_{m \to \infty} r(S_m) = \mathtt{e}_X(\alpha).
\]

The statement of Theorem~\ref{thm:Jung} consists of two parts: an exact (tight) expression for \(\OPT\) in terms of the function \(\mathtt{e}_X\), and upper and lower bounds on \(\mathtt{e}_X\).

We begin with the proof of the first part.

\begingroup
\def\thetheorem{\ref{thm:Jung}}
\begin{theorem}[First part]
Let \( X \) be a totally bounded metric space. Then, for any \( \epsilon, \delta \ge 0 \),
\[
\OPT_X(\epsilon,\delta) = \mathtt{e}_X\bigl((2+\epsilon)\delta\bigr).
\]
\end{theorem}
\addtocounter{theorem}{-2}
\endgroup
\begin{proof}

To prove the equality, we establish both inequalities:
\[
\OPT_X(\epsilon,\delta) \ge \mathtt{e}_X\bigl((2+\epsilon)\delta\bigr)
\quad \text{and} \quad 
\OPT_X(\epsilon,\delta) \le \mathtt{e}_X\bigl((2+\epsilon)\delta\bigr).
\]

\emph{Lower bound}. To show that the optimal error is \emph{at least} \(\mathtt{e}_X((2+\epsilon)\delta)\), it suffices to construct responder strategies that guarantee a reconstruction error arbitrarily close to this value. 

Although the supremum in the definition of \(\mathtt{e}_X\) may not be attained by any single set, we can approximate it by a sequence of sets \(\{S_m\}\) with \(\diam \; S_m \le (2+\epsilon)\delta\) and \(r(S_m) \longrightarrow \mathtt{e}_X((2+\epsilon)\delta)\). For each such set, we define a responder strategy that preserves \(S_m\) inside the feasible region, thereby ensuring that the reconstructor cannot achieve error smaller than \(r(S_m)\). Taking the limit yields the desired lower bound.

\emph{Upper bound}. 
To establish that the optimal error does not exceed \(\mathtt{e}_X((2+\epsilon)\delta)\), we construct a sequence of reconstruction strategies, each using a query set of size \(T_n\), such that the corresponding error remains within \(\mathtt{e}_X((2+\epsilon)\delta) + \alpha_n\), where \(\alpha_n \to 0\).

Since \(X\) is totally bounded, for any precision level \(\alpha > 0\), there exists a finite set \(T_\alpha \subset X\) that forms an \(\alpha\)-cover of the space. After querying every point in such a cover, a feasible region has diameter smaller than \((2+\epsilon)\delta+\left( (1+\epsilon)^2 +1 \right) \alpha\), and hence \[\OPT(T_\alpha, \epsilon, \delta) \le \mathtt e_X((2+\epsilon)\delta+\left( (1+\epsilon)^2 +1 \right) \alpha).\] The remaining step is to show that diameter-radius-profile \(\mathtt{e}_X\) is right-continuous, i.e.,
\[
  \mathtt{e}_X((2+\epsilon)\delta+\alpha') \xrightarrow[\alpha' \to 0]{} \mathtt{e}_X((2+\epsilon)\delta),
\]
which requires general machinery from topology—specifically, endowing the collection of compact subsets of \(X\) with a natural metric that measures how far these subsets are from each other within~\(X\). This part of the proof is deferred to \cref{app:right_continuity_e}.
 
% To argue that the feasible region with the diameter exceeding \((2+\epsilon)\delta\) by an arbitrarily small amount has the Chebyshev radius also exceeding \(\mathtt e_X((2+\epsilon)\delta)\) only by an equally small margin, we prove the general fact that the function \(\mathtt e_X\) is right-continuous for any totally bounded metric space \(X\). 

 % requires general machinery for topological spaces as endowing the compact subsets with a natural metric measuring how far these subsets are in \(X\). 

\subsection{Lower bound via extremal sets}\label{app:Jung-lower}

As mentioned earlier, the supremum 
\[
\mathtt{e}_X\left((2+\epsilon)\delta\right) := \sup_{\substack{S \subseteq X \\ \diam(S) \le (2+\epsilon)\delta}} r(S)
\]
plays the role of a maximum, although it may not actually be attained. In such cases, we simulate extremal sets—that is, sets that would attain this maximum—by considering approximately extremal sets: a sequence \(\{S_m\}_{m \in \NN}\) satisfying
\[
r(S_m) \underset{m \to \infty}{\longrightarrow} \mathtt e_X(\alpha), \qquad \diam \, S_m \le (2+\epsilon)\delta.
\]

For any \(m \in \NN\), define a responder strategy that, given any query \(q \in X\), replies with
\[
r_q := (1+\epsilon) \inf_{s \in S_m} \dist(q, s) + \delta.
\]
Here, the infimum plays the role of a minimum; so if the minimum is attained at some point \(B \in S_m\), this strategy effectively places \(B\) on the boundary of the feasible region (see Figure~\ref{fig:feasible-sets-maximal}).

Let us elaborate. We will show that for any point \(s \in S_m\), the response satisfies
\[
  r_q \le (1+\epsilon)\dist(q, s) + \delta,
\]
and then, using the triangle inequality, we will obtain the reverse bound,
\[
  r_q \ge (1+\epsilon)\dist(q, s) - \delta.
\]

The inequality \(r_q \le (1+\epsilon)\dist(q, s) + \delta\) follows directly from the definition of the infimum, which represents the minimal possible distance:
\[
  (1+\epsilon)\dist(q, s) + \delta
  \ge (1+\epsilon)\inf_{x \in S_m} \dist(q, x) + \delta
  = r_q.
\]

For the reverse direction, express \(\inf_{y \in S_m} \dist(q, y)\) in terms of \(r_q\):
\[
\inf_{y \in S_m} \dist(q, y) = \frac{r_q - \delta}{1+\epsilon}.
\]

By the triangle inequality, for any two points \(y, s \in S_m\), we have
\[
\dist(s, q) \le \dist(y, q) + \dist(s, y).
\]
Since \(\dist(s, y) \le \diam \; S_m \le (2+\epsilon)\delta\), it follows that
\[
\dist(s, q) \le \dist(y, q) + (2+\epsilon)\delta.
\]

Combining this with the inequality \((1+\epsilon)^2 \dist(y, q) \ge \dist(y, q)\), and taking the infimum over \(y \in S_m\), we obtain that for any point \(s \in S_m\),
\[
(1+\epsilon) r_q + \delta = (1+\epsilon)^2 \inf_{y \in S_m} \dist(q, y) + (2+\epsilon)\delta \ge \dist(s, q).
\]

Therefore, \(r_q =_{\epsilon,\delta} \dist(s, q)\) for every point \(s \in S_m\), and hence the entire set \(S_m\) lies within the feasible region. Once the reconstructor selects a guess point \(\hat{x}\), the responder may choose any point from the feasible region, and in particular any \(s \in S_m\).

By the definition of the Chebyshev radius,
\[
  r(S_m) \le \sup_{x^* \in S_m} \dist(x^*, \hat{x}),
\]
so for any \(\alpha > 0\), the responder can choose a point \(x^* \in S_m\) such that \(\dist(\hat{x}, x^*) > r(S_m) - \alpha.\)

It follows that
\[
\OPT(T, \epsilon, \delta) \ge r(S_m), \quad \text{and} \quad r(S_m) \underset{m \to \infty}{\longrightarrow} \mathtt{e}_X((2+\epsilon)\delta).
\]
Therefore,
\[
\OPT(T, \epsilon, \delta) \ge \mathtt{e}_X((2+\epsilon)\delta).
\]

\subsection{Upper bound via $\alpha$-covers}\label{app:Jung-upper}

To show that the optimal error is \emph{at most} \(\mathtt{e}_X((2+\epsilon)\delta)\), it suffices to construct a sequence of reconstructor strategies, each using \(T_n\) queries, that guarantee a reconstruction error of at most \(\mathtt{e}_X((2+\epsilon)\delta) + \alpha_n\), where \(\alpha_n \to 0\).

Since the space \(X\) is totally bounded, for any \(\alpha > 0\) there exists a finite \(\alpha\)-cover \(T_\alpha \subset X\), consisting of \(T_\alpha\) points. 

Take a sequence of \(\alpha_n\)-nets with \(\alpha_n \to 0\), and denote the number of queries in the corresponding nets by \(T_n := |T_{\alpha_n}|\). Since $X$ is totally bounded, these finite nets exist.
    
    Denote the points of the \(\alpha_n\)-net by $\{ q_t\}_{t \in [T_n]}$, and the responses of the responder by \(\{r_t\}_{t \in [T_n]}\).
    We claim that 
    \[
    \diam(\Phi(\{ q_t, r_t\}_{t \in [T_n]})) \le (2+\epsilon)\delta + ((1+\epsilon)^2+1)\alpha_n.
    \]
    To see this, take \emph{any} two points $A, B$ in the feasible region after the interaction. 
    
    There exists a query \(q \in \{q_t\}_{t \in [T_n]}\) such that \(\dist(A, q) \le \alpha_n\). 
    Let \(r\) be the responder’s answer to this query. 
    Since \(A \in \Phi(q, r)\) and \(\dist(q, A) \le \alpha_n\), we have
    \[
  r \le (1+\epsilon)\alpha_n + \delta.
    \]
    On the other hand, since \(B \in \Phi(q, r)\), we have
    \[
  \dist(q, B) \le (1+\epsilon)r + \delta 
  \le (1+\epsilon)^2\alpha_n + (2+\epsilon)\delta. 
    \]
    By the triangle inequality,
    \[
  \dist(A, B) 
  \le \dist(q, B) + \dist(A, q) 
  \le \bigl((1+\epsilon)^2 + 1\bigr)\alpha_n + (2+\epsilon)\delta.
    \]
    Hence \(\diam \, \Phi(\{q_t, r_t\}_{t \in [T_n]}) \le (2+\epsilon) \delta + ((1+\epsilon)^2+1) \alpha_n\).

    Denote by \(\alpha'_n\) the quantity \(((1+\epsilon)^2 + 1)\alpha_n\). 
    The Chebyshev radius of \(\Phi\bigl(\{(q_t, r_t)\}_{t \in [T_n]}\bigr)\) is therefore bounded by \(\mathtt{e}_X\bigl((2+\epsilon)\delta + \alpha'_n\bigr)\), and hence
    \[
  \OPT_X(T_n, \epsilon, \delta) \le \mathtt{e}_X\bigl((2+\epsilon)\delta + \alpha'_n\bigr).
    \]
    By the right-continuity of \(\mathtt{e}_X\) (see \cref{app:right_continuity_e}), for every sequence of nonnegative numbers \(\alpha'_n \to 0\) we have
\[
  \mathtt{e}_X\bigl((2+\epsilon)\delta + \alpha'_n\bigr) 
  \;\longrightarrow\; 
  \mathtt{e}_X\bigl((2+\epsilon)\delta\bigr).
\]
Therefore, we conclude the desired bound
\[
  \OPT_X(\epsilon, \delta) \le \mathtt{e}_X\bigl((2+\epsilon)\delta\bigr).
\]

\end{proof}

The proof of the second part of the theorem relies on general properties of the function \(\mathtt{e}_X\) that hold for arbitrary metric spaces.

\begingroup
\def\thetheorem{\ref{thm:Jung}}
\begin{theorem}[Second part]
If the distance \( (2+\epsilon)\delta \) is realized in a totally bounded metric space \( X \), i.e., there exist a pair of points at this distance, then
\[
\frac{1}{2}(2+\epsilon)\delta \leq \OPT_X(\epsilon,\delta) \leq (2+\epsilon)\delta.
\]
\end{theorem}
\addtocounter{theorem}{-2}
\endgroup
\begin{proof}
By the first part of Theorem~\ref{thm:Jung}, which we proved earlier,
\[
\OPT_X(\epsilon, \delta) = \mathtt{e}_X((2+\epsilon)\delta).
\]
So it suffices to prove that
\[
\frac{1}{2}(2+\epsilon)\delta \le \mathtt{e}_X((2+\epsilon)\delta) \le (2+\epsilon)\delta.
\]
To show the lower bound \(\mathtt{e}_X((2+\epsilon)\delta) \ge \frac{1}{2}(2+\epsilon)\delta\), it suffices to construct a set \(S \subseteq X\) of diameter at most \((2+\epsilon)\delta\) such that \emph{every} enclosing ball of \(S\) must have radius at least \(\frac{1}{2}(2+\epsilon)\delta\).
Indeed, by assumption, there exist two points \(y_1, y_2 \in X\) such that
\[
\dist(y_1, y_2) = (2+\epsilon)\delta.
\]
Then for any point \(x \in X\), the triangle inequality implies
\[
\dist(y_1, x) + \dist(x, y_2) \ge (2+\epsilon)\delta,
\]
so one of the two distances must be at least \(\frac{1}{2}(2+\epsilon)\delta\).
Therefore, no point in \(X\) lies at a distance less than \(\tfrac{1}{2}(2+\epsilon)\delta\) from both \(y_1\) and \(y_2\), and thus any ball containing both points must have radius at least this value.
To show the upper bound \(\mathtt{e}_X((2+\epsilon)\delta) \le (2+\epsilon)\delta\), it suffices to find an enclosing ball of radius \((2+\epsilon)\delta\) for \emph{any} set \(S \subseteq X\) of diameter at most \((2+\epsilon)\delta\).
Indeed, let \(x \in S\) be any point of the set, and consider the ball of radius \((2+\epsilon)\delta\) centered at \(x\). Since the diameter of \(S\) is at most \((2+\epsilon)\delta\), every point \(y \in S\) satisfies \(\dist(x, y) \le \diam(S) \le (2+\epsilon)\delta\), so \(S\) is entirely contained in this ball. This proves the claim.
\end{proof}

%-----------------------------------------------------------
\section{Feasible-region calculus}\label{app:feasible}
%-----------------------------------------------------------

The goal of this section is to determine when there exists an answer \(r \in \RR_+\) such that a given set \(S \subset X\) is contained in the feasible region \(\Phi(q, r)\) (see Equation \ref{eq:feasible_region}) resulting from a query at point \(q\). 

We will answer this question and provide a criterion for such an answer in Lemma \ref{lem:rminmax}.

For any set \(S \subset X\), define its \(\alpha\)-neighborhood by
\[
S_\alpha := \bigcup_{x \in S} B(x, \alpha),
\]
where \(B(x, \alpha)\) denotes the ball of radius \(\alpha\) centered at \(x\).

We will also be interested in the following optimization problem: what is the largest value of \(\alpha\) such that the \(\alpha\)-neighborhood of a fixed set \(S \subset X\) can be entirely contained in the feasible region for some answer to the query \(q\)? We will describe this quantity for convex Euclidean subspaces and specify the answer that the responder must give in order to preserve this neighborhood within the feasible region.

To answer the first question, it is useful to consider two natural candidates for the answer:
\[
r^{\min}_q(S) := \frac{\sup_{s \in S} \dist(s, q) - \delta}{1+\epsilon}, 
\qquad
r^{\max}_q(S) := (1+\epsilon) \inf_{s \in S} \dist(s, q) + \delta.
\]
Intuitively, \(r^{\min}_q(S)\) places the farthest point of \(S\) on the outer boundary of the feasible region while keeping all of \(S\) inside it, and \(r^{\max}_q(S)\) places the nearest point of \(S\) on the inner boundary while still preserving inclusion (see \cref{fig:feasible-sets-minimal}, \cref{fig:feasible-sets-maximal}).

Supremum and infimum of the set of numbers \(\{\dist(y, s)\}_{s \in S}\) play the same role as \(\min_{s \in S} \dist(s, q)\) and \(\max_{s \in S} \dist(s, q)\). The only difference is that sometimes the minimum or maximum is not attained by any point in the set \(S\). In such cases, one must take a sequence of points \(\{x_i\}_{i \in \mathbb{N}} \subseteq S\) that plays the role of the minimum or maximum, in the sense that
\[
\lim_{i \to \infty} \dist(y, x_i) = \sup_{s \in S} \dist(y, s)
\quad \text{or} \quad
\lim_{i \to \infty} \dist(y, x_i) = \inf_{s \in S} \dist(y, s).
\]

\begin{lemma}[Consistency window]\label{lem:rminmax}
Fix a set \(S \subset X\). For a given query \(q\), there exists an answer \(r\) such that \(S\) is contained in the feasible region \(\Phi(q, r)\) if and only if \(r^{\min}_q(S) \le r^{\max}_q(S)\).  
Moreover, this inclusion holds if and only if \(r \in [\,r^{\min}_q(S),\, r^{\max}_q(S)]\)\footnote{Note that we do not require the answer to be positive; it may be negative, yet the feasible region can still be non-empty.}.
\end{lemma}

\begin{proof}
Assume the responder gives an answer \(r\) such that \(S \subset \Phi(q, r)\).  
By the definition of the supremum, if
\[
r < \frac{\sup_{x \in S} \dist(x, q) - \delta}{1 + \epsilon},
\]
then there exists a point \(A \in S\) such that
\[
r < \frac{\dist(A, q) - \delta}{1 + \epsilon},
\]
and hence \(A \notin \Phi(q, r)\), contradicting the assumption.

Similarly, if \(r > r^{\max}_q(S) = (1 + \epsilon) \inf_{s \in S} \dist(s, q) + \delta\), then there exists a point \(B \in S\) such that
\[
r > (1 + \epsilon) \dist(B, q) + \delta,
\]
and therefore \(B \notin \Phi(q, r)\).

Hence, for any \(r\) outside the interval \([r^{\min}_q(S), r^{\max}_q(S)]\), there exists a point in \(S\) that lies outside \(\Phi(q, r)\). This shows that if \(S \subset \Phi(q, r)\), then necessarily \(r \in [r^{\min}_q(S), r^{\max}_q(S)]\), and in particular \(r^{\min}_q(S) \le r^{\max}_q(S)\).

To prove the converse, suppose \(r \in [r^{\min}_q(S), r^{\max}_q(S)]\). Take any point \(s \in S\). We must verify the two inequalities:
\[
\dist(s, q) \le (1 + \epsilon)r + \delta \quad \text{and} \quad r \le (1 + \epsilon) \dist(s, q) + \delta.
\]

Indeed, since
\[
r \le r^{\max}_q(S) = (1 + \epsilon) \inf_{x \in S} \dist(x, q) + \delta \le (1 + \epsilon) \dist(s, q) + \delta,
\]
and
\[
r \ge r^{\min}_q(S) = \frac{\sup_{x \in S} \dist(x, q) - \delta}{1 + \epsilon} \ge \frac{\dist(s, q) - \delta}{1 + \epsilon},
\]
we have \(s \in \Phi(q, r)\). Since \(s \in S\) was arbitrary, it follows that \(S \subset \Phi(q, r)\), completing the proof.
\end{proof}

The observation above does not rely on any structural properties of the metric space; in particular, it holds even if the triangle inequality is not satisfied.
However, to determine the largest neighborhood of a set that may remain feasible after a query, we need to use some form of continuity in the space. That’s why, from this point on, we assume that the metric space \(X\) is a convex subset of $\RR^n$.

\begin{observation}
\label{ob:r_neighborhood}
Let \( S \subseteq X \), and let \( \alpha > 0 \). Suppose that for every \( x \in S \), the Euclidean ball \( B(x, \alpha) \subseteq \mathbb{R}^n \) is entirely contained in \( X \). Then:
\[
r^{\min}_q(S_\alpha) = r^{\min}_q(S) + \frac{\alpha}{1+\epsilon},
\qquad
r^{\max}_q(S_\alpha) = \max\left\{\delta, \, r^{\max}_q(S) - \alpha(1+\epsilon)\right\}.
\]
\end{observation}

\begin{proof}
We start by analyzing the supremum. We want to show:
\[
\sup_{y \in S_\alpha} \dist(q, y) = \sup_{s \in S} \dist(s, q) + \alpha.
\]

The inequality
\[
\sup_{y \in S_\alpha} \dist(q, y) \le \sup_{s \in S} \dist(s, q) + \alpha
\]
is immediate from the triangle inequality. Indeed, for any point \( y \in S_\alpha \), there exists some \( s \in S \) such that \( y \in B(s, \alpha) \). Then:
\[
\dist(q, y) \le \dist(q, s) + \dist(s, y) \le \sup_{s \in S} \dist(s, q) + \alpha,
\]
and so the inequality holds for all \( y \in S_\alpha \), yielding the upper bound on the supremum.

For the reverse inequality, take a sequence \( \{x_i\}_{i \in \mathbb{N}} \subset S \) such that:
\[
\lim_{i \to \infty} \dist(x_i, q) = \sup_{s \in S} \dist(s, q).
\]
For each \( x_i \), choose a point \( y_i \in B(x_i, \alpha) \) lying along the ray from \( q \) through \( x_i \) such that:
\[
\dist(q, y_i) = \dist(q, x_i) + \alpha.
\]
This is possible because the balls are Euclidean. Then:
\[
\sup_{y \in S_\alpha} \dist(q, y) \ge \lim_{i \to \infty} \dist(q, y_i) = \sup_{s \in S} \dist(s, q) + \alpha,
\]
proving the desired equality.

Now we turn to the infimum:
\[
\inf_{y \in S_\alpha} \dist(q, y) = \max\left\{0, \inf_{s \in S} \dist(s, q) - \alpha \right\}.
\]

First, by the triangle inequality again, we have:
\[
\dist(q, y) \ge \dist(q, s) - \dist(y, s) \ge \inf_{s \in S} \dist(s, q) - \alpha
\]
for all \( y \in S_\alpha \). Also, clearly \( \dist(q, y) \ge 0 \). Therefore,
\[
\inf_{y \in S_\alpha} \dist(q, y) \ge \max\left\{0, \inf_{s \in S} \dist(s, q) - \alpha \right\}.
\]

To show the reverse inequality, consider a sequence \( \{x_i\}_{i \in \mathbb{N}} \subset S \) such that:
\[
\lim_{i \to \infty} \dist(x_i, q) = \inf_{s \in S} \dist(s, q).
\]

If \( \inf_{s \in S} \dist(s, q) < \alpha \), then for some \( x_j \), we have \( \dist(x_j, q) < \alpha \), so \( q \in B(x_j, \alpha) \subset S_\alpha \), and thus:
\[
\inf_{y \in S_\alpha} \dist(q, y) = 0.
\]

Otherwise, all \( x_i \) satisfy \( \dist(x_i, q) \ge \alpha \). In that case, for each \( x_i \), there exists a point \( y_i \in [q, x_i] \) such that \( \dist(x_i, y_i) = \alpha \), i.e., \( y_i \) lies along the segment from \( q \) to \( x_i \), at distance \( \alpha \) from \( x_i \). Then:
\[
\dist(y_i, q) = \dist(x_i, q) - \alpha,
\]
and so:
\[
\inf_{y \in S_\alpha} \dist(q, y) \le \lim_{i \to \infty} \dist(y_i, q) = \inf_{s \in S} \dist(s, q) - \alpha.
\]

This completes the proof.
\end{proof}

% Notice that as \(\alpha\) increases, the difference \(r_q^{\max}S_\alpha - r_q^{\min}S_\alpha\) decreases (and may become negative).  This raises the natural question: what is the largest \(\alpha\) for which  
% \[
% r_q^{\max}S_\alpha - r_q^{\min}S_\alpha \;\ge\; 0?
% \]  
% Equivalently, what is the largest \(\alpha\) such that there exists a response \(r\) with  
% \[
% S_\alpha \;\subset\; \Phi(q,r)?
% \]

Let us denote
\[
  \rho^q_{\min}(S) := \inf_{s \in S} \dist(q, s),
  \qquad
  \rho^q_{\max}(S) := \sup_{s \in S} \dist(q, s).
\]
These quantities represent the minimal and maximal distances from the query point \(q\) to the set \(S\), and will be used to simplify the expressions that follow.

\begin{lemma}
\label{lem:biggest_neighborhood}
Fix a set \(S \subset X\) and a query point \(q \in X\). Define
\[
\alpha^\star = \frac{r^{\max}_q(S) - r^{\min}_q(S)}{(1+\epsilon) + \frac{1}{1+\epsilon}}.
\]
Assume that \(\alpha^\star > 0\), and that for every point \(s \in S\), the Euclidean ball \(B(s, \alpha^\star)\), viewed as a subset of \(\mathbb{R}^n\), is contained in \(X\); that is, \(B(s, \alpha^\star) \subseteq X\).

Then there exists an answer \(r\) such that \(S_{\alpha^\star} \subset \Phi(q, r)\). In particular, for the specific choice
\[
r^\star_q(S) := \frac{(1+\epsilon)\big(\rho^q_{\min} + \rho^q_{\max}\big) - \epsilon\delta}{1 + (1+\epsilon)^2},
\]
we have \(S_{\alpha^\star} \subset \Phi(q, r^\star_q(S))\).
\end{lemma}
\begin{proof}

By Lemma~\ref{lem:rminmax}, it suffices to verify that
\[
r^{\min}_q(S_{\alpha^\star}) \le r^\star_q(S) \le r^{\max}_q(S_{\alpha^\star}).
\]

First, observe that
\[
\rho^q_{\max} = (1+\epsilon) r_q^{\min}(S) + \delta, \qquad \rho^q_{\min} = \frac{r^{\max}_q(S)-\delta}{1 + \epsilon}
\]
and therefore
\[
r^\star_q(S) = \frac{r_q^{\max} + (1+\epsilon)^2r^{\min}_q}{(1+\epsilon)^2+1}.
\]

To verify the lower bound, note that by Observation~\ref{ob:r_neighborhood},
\[
\begin{aligned}
r^{\min}_q(S_{\alpha^\star}) &= r^{\min}_q(S)+ \frac{\alpha^\star}{1 + \epsilon} \\
&= r^{\min}_q(S)+ \frac{r^{\max}_q(S) - r^{\min}_q(S)}{(1 + \epsilon)^2+1} \\
&= \frac{(1+\epsilon)^2 \cdot r_q^{\min}(S) + r_q^{\max}(S)}{(1+\epsilon)^2 +1} \\
&= r^\star_q(S).
\end{aligned}
\]

For the upper bound, we distinguish between two cases depending on whether \(\rho^q_{\min} \ge \alpha^\star\) or \(\rho^q_{\min} < \alpha^\star\). Indeed, the form of \(r^{\max}_q(S_{\alpha^\star})\) is case-dependent, with two distinct formulas:
if \(\rho^q_{\min} \ge \alpha^\star\), then \(r^{\max}_q(S_{\alpha^\star}) = r^{\max}_q(S) - (1+\epsilon) \alpha^\star\); otherwise, \(r^{\max}_q(S_{\alpha^\star}) = \delta\).

\textbf{Case} \(\rho^q_{\min} \ge \alpha^\star\). Then by Observation~\ref{ob:r_neighborhood},
\[
\begin{aligned}
r^{\max}_q(S_{\alpha^\star}) &= r^{\max}_q(S) - (1 + \epsilon)\alpha^\star \\
&= r^{\max}_q(S) - \frac{(1+\epsilon)^2 \left(r^{\max}_q(S) - r^{\min}_q(S)\right)}{(1 + \epsilon)^2+1} \\
&= \frac{(1+\epsilon)^2 \cdot r_q^{\min}(S) + r_q^{\max}(S)}{(1+\epsilon)^2 +1} \\
&= r^\star_q(S).
\end{aligned}
\]

\textbf{Case} \(\rho^q_{\min} < \alpha^\star\). In this case, we observe that
\[
r^{\max}_q(S) = (1+\epsilon) \cdot \rho^q_{\min} + \delta < (1+\epsilon)\alpha^\star + \delta,
\]
and therefore
\[
\delta > r^{\max}_q(S) - (1+\epsilon)\alpha^\star = r^\star_q(S),
\]
so again \(r^\star_q(S) < \delta = r^{\max}_q(S_{\alpha^\star})\), as required.

This completes the proof.
\end{proof}

For later use, we express the quantity \(\alpha^\star\) in terms of \(\rho^q_{\min}\) and \(\rho^q_{\max}\), since this representation will be useful below:
\[
  \alpha^\star 
  = \frac{(1+\epsilon)^2 \rho^q_{\min} - \rho^q_{\max} + (2+\epsilon)\delta}
         {(1+\epsilon)^2 + 1}.
\]
This leads to the following observation.

\begin{remark}
\label{rem:neighb_decomposition}
Assume that \(\diam \,S \le (2+\epsilon)\delta\). Then the radius \(\alpha^\star\) of the neighborhood \(S_{\alpha^\star}\), as defined in Lemma~\ref{lem:biggest_neighborhood}, can be decomposed into two nonnegative terms, \(\alpha^\star = \alpha_1 + \alpha_2\), where
\[
\alpha_1 = \frac{(2+\epsilon)\delta - \left( \sup_{x \in S} \dist(x, q) - \inf_{x \in S} \dist(x, q) \right)}{(1+\epsilon)^2 + 1}
\]
and
\[
\alpha_2 = \frac{(1+\epsilon)^2 - 1}{(1+\epsilon)^2 + 1} \cdot \inf_{x \in S} \dist(x, q).
\]
Note that \(\alpha_1 \ge 0\), since 
\(\sup_{x \in S}\dist(x, q) - \inf_{x \in S}\dist(x, q) 
   \le \diam S \le (2+\epsilon)\delta.\)
\end{remark}

\begin{observation}
Assume \(\epsilon = 0\) and fix the regular simplex \( \Delta \subset X\) with edges of length \(2 \delta\), and the query \(q \in \RR^n\). 
Denote 
\[
A = \argmax_{A_i \in \Delta} \dist(A, q), \qquad B = \argmin_{A_i \in \Delta} \dist(A, q).
\]

Then Lemma~\ref{lem:biggest_neighborhood} can be simplified; 
the \(\alpha^\star\)-neighborhood lies in the feasible region: \(\Delta_{\alpha^\star} \subset \Phi(q, r^\star_q(\Delta))\) for 
\begin{align*}
r^\star_q(\Delta)=\frac{\dist(q, A) +\dist(q,B)}{2}, \\
\alpha^\star = \frac{\dist(q, A) - \dist(q,B)}{2}.
\end{align*}
\end{observation}

\begin{lemma}\label{lem:criteria_good}
Assume we are in the \(\epsilon=0\) game scenario. Fix \(\delta > \alpha > 0\), and let \(\Delta = A_0A_1 \ldots A_n\) be a regular simplex with edges of length \(2\delta\). 
For a given point \(q\), let \(A \in \Delta\) be the farthest vertex from \(q\), and let \(B \in \Delta\) be the nearest vertex 
(in the case of ties, choose any). 

If
\[
\cos \angle BAq \,\le\, 1 - \frac{\alpha}{\delta},
\]
then
\[
\Delta_\alpha \,\subset\, \Phi\big(\{q, r^\star_q(\Delta)\}\big), \qquad r^\star_q(\Delta) = \frac{qA + qB}{2}.
\]
\end{lemma}

\begin{proof}
It suffices to show that \(B(A, \alpha)\) and \(B(B, \alpha)\) are both contained in \(\Phi(q, r^\star_q(\Delta))\).

Let us estimate \(Aq - Bq\). We will show that 
\[ Aq - Bq \le \cos \angle BAq \cdot 2\delta.\] 
Drop a perpendicular from \(q\) onto the line \(AB\), and let the foot of this perpendicular be \(H\). Denote the angle \(\angle qAH\) by \(\phi\) and the angle \(\angle qBH\) by \(\psi\). 

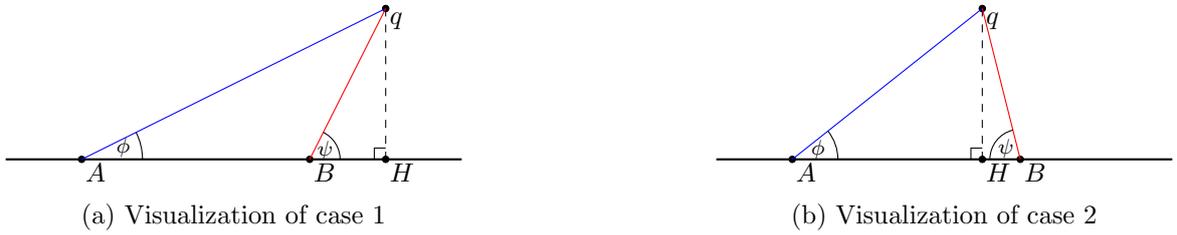
\begin{figure}[ht]
  \centering
  \begin{subfigure}[b]{0.45\textwidth}
    \centering
    \begin{tikzpicture}[scale=1,>=stealth]
      %---- coordinates ----
      \coordinate (A) at (0,0);
      \coordinate (B) at (3,0);
      \coordinate (q) at (4,2);
      \coordinate (H) at (4,0);
      %---- the line AB (extended) ----
      \draw[thick] (-1,0) -- (5,0);
      %---- perpendicular drop ----
      \draw[dashed] (q) -- (H);
      %---- right‐angle marker at H ----
      \draw (H) ++(-0.15,0) -- ++(0,0.15) -- ++(0.15,0);
      %---- points and labels ----
      \foreach \pt/\lbl in {A/A, B/B, H/H, q/q}{
        \draw[fill] (\pt) circle(1.2pt)
          node[below right, inner sep=1.3pt, font=\small] {$\lbl$};
      }
      \draw[color = blue] (A) -- (q);
      \draw[color = red] (B) -- (q);
      \pic[draw,angle radius=8mm,
       "$\phi$"{font=\scriptsize, xshift=2.2pt,yshift=1.25pt} ] {angle = H--A--q};
       \pic[draw,angle radius=4mm,
       "$\psi$"{font=\scriptsize}] {angle = H--B--q};
    \end{tikzpicture}
    \caption{Visualization of case 1}
  \end{subfigure}%
  \hfill
  \begin{subfigure}[b]{0.45\textwidth}
    \centering
    \begin{tikzpicture}[scale=1,>=stealth]
      %---- coordinates ----
      \coordinate (A) at (0,0);
      \coordinate (B) at (3,0);
      \coordinate (q) at (2.5,2);
      \coordinate (H) at (2.5,0);
      %---- the line AB (extended) ----
      \draw[thick] (-1,0) -- (5,0);
      %---- perpendicular drop ----
      \draw[dashed] (q) -- (H);
      %---- right‐angle marker at H ----
      \draw (H) ++(-0.15,0) -- ++(0,0.15) -- ++(0.15,0);
      %---- points and labels ----
      \foreach \pt/\lbl in {A/A, B/B, H/H, q/q}{
        \draw[fill] (\pt) circle(1.2pt)
          node[below right, inner sep=1.3pt, font=\small] {$\lbl$};
      }
      \draw[color = blue] (A) -- (q);
      \draw[color = red] (B) -- (q);
      \pic[draw,angle radius=6mm,
       "$\phi$"{font=\scriptsize, xshift=0pt,yshift=0.5pt}] {angle = H--A--q};
       \pic[draw,angle radius=4mm,
       "$\psi$"{font=\scriptsize}] {angle = q--B--H};
    \end{tikzpicture}
    \caption{Visualization of case 2}
  \end{subfigure}
  \caption{Visualization of two cases}
  \label{fig:two_cases_criteria}
\end{figure}

We distinguish two cases: \(H \notin AB\) and \(H \in AB\) (see Figure~\ref{fig:two_cases_criteria}). Note that in both cases, \[0 < \phi \le \psi < \frac{\pi}{2},\] and since the cosine function decreases on this interval, we have \(\cos \phi \ge \cos \psi \ge 0\).

\textbf{Case 1:} \(H \notin AB\). Since \(Aq^2 = qH^2+HA^2\) and \(Bq^2 = BH^2+HB^2\), and \(AH -BH = 2 \delta \) we compute:
\[
Aq - Bq \;=\; \frac{Aq^2 - Bq^2}{Aq + Bq}
\;=\; \frac{AH^2 - BH^2}{Aq + Bq}
\;=\; 2\delta \,\frac{AH + BH}{Aq + Bq}.
\]
On the other hand since \(\cos \phi \ge \cos \psi\): 
\[
\frac{AH + BH}{Aq + Bq} = \frac{Aq \cdot \cos \phi + Bq \cdot \cos \psi}{Aq+Bq} \le 1. 
\]

Hence
\[
Aq - Bq \;\le\; 2\delta\,\cos \phi = 2\delta\,\cos \angle qAB.
\]

\textbf{Case 2:} \(H \in AB\). Similarly since \(AH + BH = 2 \delta \): 
\[
\begin{aligned}
Aq - Bq \;=\; \frac{Aq^2 - Bq^2}{Aq + Bq}
\;=\; \frac{AH^2 - BH^2}{Aq + Bq}
\;=\; 
2\delta \,\frac{AH - BH}{Aq + Bq} = \\2\delta \,\frac{Aq \cdot \cos \phi - Bq \cdot \cos \psi}{Aq + Bq},
\end{aligned}
\]
which again implies
\[
Aq - Bq \;\le\; 2\delta\,\cos \phi = 2\delta\,\cos \angle qAB.
\]

Now take a point \(M_A \in B(A, \alpha)\). By Lemma assumption \(\alpha \le \delta(1 - \cos \angle BAq)\), hence 
\[
qM_A \;\le\; qA + \alpha 
\;\le\; qA + \delta(1 - \cos \angle BAq)
\;\le\; qA - \frac{Aq - Bq}{2} + \delta
\;=\; r^\star_q(\Delta) + \delta.
\]
Also, because \(\alpha < \delta\), it follows that \(M_AA \ge r^\star_q(\Delta)\). Hence \(M_A \in \Phi(q, r^\star_q(\Delta))\). 

A similar argument applies to any \(M_B \in B(B, \alpha)\). Indeed,
\[
qM_B \;\ge\; qB - \alpha
\;\ge\; qB - \delta (1 - \cos \angle BAq)
\;\ge\; r^\star_q(\Delta) - \delta,
\]
and again \(\alpha < \delta\) implies \(M_B \in \Phi(q, r^\star_q(\Delta))\). Thus \(B(A, \alpha)\) and \(B(B, \alpha)\) lie in \(\Phi(q, r^\star_q(\Delta))\), completing the proof.
\end{proof}

%-----------------------------------------------------------
\section{Proof of Theorem~\ref{thm:pseudo-finiteness}}\label{app:pseudo}
%-----------------------------------------------------------

\begin{theorem*}[Theorem~\ref{thm:pseudo-finiteness} Restatement]
Let \(X \subset \mathbb{R}^n\) be a bounded convex set equipped with the Euclidean metric such that \(\dim X > 0\), and let \(\epsilon \ge 0\). Then, for all sufficiently small \(\delta > 0\), the space \(X\) is \emph{not} \((\epsilon, \delta)\)-pseudo‑finite, except in the case where \(\epsilon = 0\) and \(\dim X = 1\).
\end{theorem*}

\paragraph{Proof.}

Since the reconstruction game is played entirely within the space \(X\), in the sense that all
queries, the final guess, and the secret point lie in \(X\), we may assume without loss of generality that
\(X \subset \mathbb{R}^n\), where \(n = \dim(X)\) is the affine dimension of \(X\).

We begin with the special and simple case where \(\epsilon = 0\) and \(\dim X = 1\). In this case, \(X\) is an interval with endpoints \(a, b\); without loss of generality, we assume that both \(a, b \in X\). (The cases where~\(X\) is half-open or open can be handled similarly.) If \(\delta \ge (b - a)/2\), then the optimal error equals the diameter of the space, and the reconstructor can trivially achieve it by outputting any point in \(X\) without submitting any queries. Otherwise, if \(\delta < (b - a)/2\), the reconstructor submits a single query at one of the endpoints, say \(q_1 = a\), and receives a \(r_1)\). The feasible region then becomes an interval of length at most \(2\delta\), and by guessing its midpoint, the reconstructor achieves the optimal approximation error of \(\delta\).

The remaining cases—when either \(\epsilon > 0\) or \(\dim X \ge 2\)—are more challenging and constitute the core of the proof.
% Since the reconstruction game is played entirely within the space \(X\)—in the sense that all queries, the final guess, and the secret point all lie in \(X\)—we may assume without loss of generality that \(X \subset \mathbb{R}^n\), where \(n = \dim(X)\) is the affine dimension of \(X\).
We divide the proof into two cases: one where \(\epsilon = 0\), and one where \(\epsilon > 0\). In both cases, the proof follows a similar strategy. We show that for all sufficiently small \(\delta > 0\), there exists a responder strategy that guarantees, for every number of rounds \(T\), that the feasible region contains an extremal simplex \(\Delta_T\) whose Chebyshev radius is strictly greater than the optimal value \(\OPT(\epsilon, \delta)\). This suffices to prove that the optimal error cannot be attained in finite time.

More precisely, we show that for each \(t = 0,1,\dots,T\), the responder can ensure that the feasible region contains an \(\alpha_t\)-neighborhood of a regular simplex \(\Delta_t\) of diameter exactly \((2+\epsilon)\delta\), where the neighborhood is defined as the union of all balls of radius \(\alpha_t\) centered at the vertices of \(\Delta_t\). Since such a neighborhood contains a regular simplex of diameter \((2+\epsilon)\delta + \sqrt{\tfrac{2(n+1)}{n}} \alpha_t\), it follows that the Chebyshev radius of the feasible region is strictly greater than \(\OPT(\epsilon,\delta)\), which corresponds to the Chebyshev radius of a regular simplex of diameter \((2+\epsilon)\delta\).

The proof proceeds inductively. We assume that after \(t\) rounds, the feasible region contains an \(\alpha_t\)-neighborhood of a regular simplex \(\Delta_t\) of diameter \((2+\epsilon)\delta\), and we show that for any query \(q_{t+1} \in X\), there exists a response such that the updated feasible region contains an \(\alpha_{t+1}\)-neighborhood of some (possibly different) regular simplex \(\Delta_{t+1}\) of the same diameter. Moreover, \(\alpha_{t+1} = f_{\epsilon,\delta}(\alpha_t)\), where the function \(f_{\epsilon,\delta}\) is defined by:
\[
f_{\epsilon,\delta}(\alpha) =
\begin{cases}
\frac{\alpha^2}{2 \cdot 81\,\delta}, & \epsilon = 0, \\
\frac{(1+\epsilon)^2 - 1}{2(1+\epsilon)^2} \cdot \alpha, & \epsilon > 0.
\end{cases}
\]

Thus, the responder can recursively maintain an \(\alpha_t\)-neighborhood of a regular simplex throughout the game, where \(\alpha_t = f_{\epsilon,\delta}^{(t)}(\alpha_0)\), and \(\alpha_0\) is the maximum value such that \(X\) contains an \(\alpha_0\)-neighborhood of a regular simplex of diameter \((2+\epsilon)\delta\). This completes the high-level argument. To complete the inductive proof, it remains to establish the base case and then develop the tools needed for the inductive step.

\medskip
\noindent\textbf{Base Case.}
Since \(X\) is a bounded convex subset of \(\mathbb{R}^n\) with nonempty interior, it follows that for any \(\epsilon > 0\), there exists a sufficiently small \(\delta > 0\) and some \(\alpha_0 > 0\) such that \(X\) contains an \(\alpha_0\)-neighborhood of a regular simplex \(\Delta^0\) with diameter \((2+\epsilon)\delta\). This establishes the base case for our induction: at round \(t = 0\), the feasible region contains a neighborhood \(\Delta^0_{\alpha_0} \subseteq X\) that satisfies the required conditions.

\medskip
\noindent\textbf{Inductive Step.} 
The remainder of the proof develops the geometric tools needed to carry out the inductive step, namely to show that such a neighborhood can be maintained (with decreasing radius) after each query.
We now introduce the some notation used throughout the argument.
\begin{notation}
\label{not:point_good}
Let \(0 < \alpha \le \delta\), and let \(\Delta = \{A_0, A_1, \ldots, A_n\}\) be a regular simplex in \(\mathbb{R}^n\) of diameter \((2+\epsilon)\delta\). We define the \(\alpha\)-neighborhood of \(\Delta\), denoted \(\Delta_\alpha\), as the union of closed Euclidean balls of radius \(\alpha\) centered at the vertices of \(\Delta\).

A query point \(q \in \mathbb{R}^n\) is called \emph{good} (with respect to \(\Delta\) and \(\alpha\)) if there exists a response \(r\) such that the entire neighborhood \(\Delta_\alpha\) is contained in the feasible region \(\Phi(q, r)\); otherwise, \(q\) is called \emph{bad}.

We will use the special response \( r = r^\star_q(\Delta) \) defined in Lemma~\ref{lem:biggest_neighborhood} to ensure that the neighborhood~\( \Delta{_{\alpha_{t+1}}} \) remains feasible; this choice will be sufficient for our inductive argument.
\end{notation}

\subsection{Case $\epsilon>0$ (translation strategy)}\label{app:pseudo-trans}

Assume the responder receives the query \(q\), and has so far managed to keep the \(\alpha\)-neighborhood of a regular simplex \(\Delta \subset \RR^n\), with edge length \((2+\epsilon)\delta\), inside the feasible region. {Without loss of generality, we may assume that \(4\alpha \le (2+\epsilon)\delta\).}

In this subsection, we will show that when \(\epsilon > 0\), there exists another regular simplex \(\Delta'\) with the same edge length such that
\[
\Delta'_{\alpha'} \subset \Delta_\alpha, \qquad \Delta'_{\alpha'} \subset \Phi(q, r^\star_q(\Delta')),
\quad \text{where} \quad \alpha' = \frac{(1+\epsilon)^2 - 1}{2(1+\epsilon)^2} \alpha.
\]

This will be sufficient to establish the induction step for the case \(\epsilon > 0\).

In the terminology of Notation~\ref{not:point_good}, the query point \(q\) may be either good or bad with respect to the simplex \(\Delta\) and neighborhood \(\alpha'\). If \(q\) is good, we are done by simply taking \(\Delta' := \Delta\).

If the point \(q\) is bad, we will use the decomposition of the neighborhood sustained by the answer~\(r^\star_q(\Delta)\), as described in Remark~\ref{rem:neighb_decomposition}, in order to construct a new simplex.

Let us remind the reader that, by Lemma~\ref{rem:neighb_decomposition}, there is a formula for the radius~\(\alpha^\star\) of the neighborhood~\(\Delta_{\alpha^\star}\), which corresponds to the feasible region after answering with \(r^\star_q(\Delta)\). Remark~\ref{rem:neighb_decomposition} states that this radius can be decomposed into two nonnegative terms.

In the case \(\epsilon > 0\), the second term will be of particular interest:
\[
\alpha^\star = \alpha_1 + \alpha_2, \qquad 
\alpha_2 = \frac{(1+\epsilon)^2 - 1}{(1+\epsilon)^2 + 1} \cdot \inf_{x \in \Delta} \dist(x, q).
\]

Since we assumed that the query \(q\) is bad with respect to the simplex \(\Delta\) and neighborhood \(\alpha'\), it follows that
\[
\alpha' > \frac{(1+\epsilon)^2 - 1}{(1+\epsilon)^2 + 1} \cdot \inf_{x \in \Delta} \dist(x, q).
\]
In particular, this implies
\[
\dist(B, q) < \alpha' \cdot \frac{(1+\epsilon)^2 + 1}{(1+\epsilon)^2 - 1}
= \frac{(1+\epsilon)^2 + 1}{2(1+\epsilon)^2} \cdot \alpha = \alpha - \alpha',
\]
where \(B := \arg\min_{x \in \Delta} \dist(x, q)\) is the closest vertex of \(\Delta\) to the query point \(q\).

Define the shifted simplex \(\Delta' := \Delta + \vec{v}\), where the vector \(\vec{v}\) is in the same direction as the vector~\(\vvv{qB}\), and its length is
\[
\|\vec{v}\| = \alpha - \alpha' 
= \frac{2(1+\epsilon)^2}{(1+\epsilon)^2 - 1} \alpha' - \alpha' 
= \frac{(1+\epsilon)^2+1}{(1+\epsilon)^2 - 1} \cdot \alpha'.
\]
In the degenerate case when \(q = B\), choose any vector \(\vec{v}\) of that length.

The shifted neighborhood satisfies \(\Delta'_{\alpha'} \subset \Delta_\alpha\). Assume \(x \in \Delta'_{\alpha'}\); then there exists a shifted vertex \(A' := A + \vec{v}\) (that is, \(A'\) is the image of \(A\) under translation by the vector \(\vec{v}\)) such that \(x \in B(A', \alpha')\).

Let us denote by \(B'\) the point obtained by shifting the vertex \(B\) of \(\Delta\) (the one closest to \(q\)) by the vector \(\vec{v}\).

We claim that \(B'\) is the nearest vertex of the translated simplex \(\Delta'\) to the query point \(q\). This follows from the triangle inequality and the bounds
\[
\dist(B, q) \le \alpha - \alpha', \qquad \alpha < \frac{(2+\epsilon)\delta}{4}
\]

Indeed, consider any other vertex \(C' = C + \vec{v}\) of \(\Delta'\), where \(C \ne B\). Then:
\[
\dist(q, C') \ge \dist(q, C) - \|\vec{v}\| = \dist(q, C) - (\alpha - \alpha').
\]
On the other hand, we have
\[
\dist(q, C) \ge \dist(C, B) - \dist(q, B) \ge (2+\epsilon)\delta - (\alpha - \alpha').
\]
Combining these inequalities gives
\[
\dist(q, C') \ge (2+\epsilon)\delta - 2(\alpha - \alpha').
\]
Meanwhile, the distance from \(q\) to \(B'\) satisfies
\[
\dist(q, B') \le \dist(q, B) + (\alpha-\alpha') \le 2(\alpha-\alpha') .
\]
Since we assumed \(\alpha < \frac{(2+\epsilon)\delta}{4}\), it follows that
\[
\dist(q, B') \le 2(\alpha - \alpha') < (2+\epsilon)\delta - 2(\alpha - \alpha') \le \dist(q, C'),
\]
which confirms that \(B'\) is indeed the closest vertex of \(\Delta'\) to \(q\).

Since \(B + \vec{v}\) is the vertex of the new simplex \(\Delta'\) closest to \(q\), we have
\[
\dist(q, B + \vec{v}) 
\;\ge\; 
\|\vec{v}\|
\;=\; 
\frac{(1+\epsilon)^2 + 1}{(1+\epsilon)^2 - 1} \cdot \alpha'.
\]
Therefore, by Remark~\ref{rem:neighb_decomposition} (which follows from Lemma~\ref{lem:biggest_neighborhood}), the entire neighborhood \(\Delta'_{\alpha'}\) is contained in the feasible region:
\[
\Delta'_{\alpha'} \subset \Phi(q, r^\star_q(\Delta')).
\]

This completes the argument.

\subsection{Case $\epsilon=0$ (rotation strategy)}\label{app:pseudo-rot}

The strategy for handling the case \(\epsilon = 0\) will be similar. Assume the responder receives a query \(q\), and that the \(\alpha\)-neighborhood of a regular simplex \(\Delta\), with edge length \(2\delta\), is contained in the feasible region. Without loss of generality we may assume that \(\alpha < \frac{\delta}{4}.\)

We will again show that there exists another regular simplex \(\Delta'\), with the same edge length, such that
\[
\Delta'_{\alpha'} \subset \Delta_\alpha, \qquad \Delta'_{\alpha'} \subset \Phi(q, r^\star_q(\Delta')),
\quad \text{where} \quad \alpha' = \frac{\alpha^2}{81 \cdot 2\delta}.
\]

In the case \(\epsilon > 0\), the \(\alpha'\)-good points were those whose distances to the simplex \(\Delta\) were sufficiently big.  
Obviously, when \(\epsilon = 0\), this method of locating good points no longer works: for example, the entire line passing through two vertices \(A\) and \(B\) of the simplex \(\Delta\) consists of \(\alpha\)-bad points for any~\(\alpha > 0\).

Note also that in our earlier argument—where we moved the simplex so that a previously bad point would become good with respect to the shifted simplex—we did not require a full characterization of bad points.
It was enough to identify a property shared by all bad points and then move the simplex so that the given point no longer satisfies that property.

The same strategy applies in the case \(\epsilon = 0\), using Lemma~\ref{lem:criteria_good}.  
Let \( B := \arg\min_{A_i \in \Delta} \dist(q, A_i) \) be the nearest vertex of \(\Delta\) to the query point \(q\), and let \( A := \arg\max_{A_i \in \Delta} \dist(q, A_i) \) be the farthest.

Lemma~\ref{lem:criteria_good} may be equivalently stated as follows: if the point \(q\) is classified as bad, then the angle between the vectors \(\vvv{AB}\) and \(\vvv{Aq}\), specifically the angle \(\angle qAB\), must be sufficiently small. The justification for this reformulation is provided below.

Since \(B\) is closer to \(q\) than \(A\), the angle \(\angle qAB\) necessarily lies within the interval \([0, \pi/2)\).
Within this interval, the cosine function decreases monotonically from \(1\) to \(0\). Consequently, if the point \(q\) is \(\alpha'\)-bad, Lemma~\ref{lem:criteria_good} implies that
\[
\cos(\vvv{AB}, \vvv{Aq}) > 1 - \frac{\alpha'}{\delta}
\quad \Longrightarrow \quad
\angle BAq < \arccos\left(1 - \frac{\alpha'}{\delta}\right).
\]

The regions consisting of points satisfying 
\(\angle BAq < \arccos\left( 1 - \frac{\alpha'}{\delta} \right)\) are illustrated in Figure~\ref{fig:criteria_visualization}.

\begin{figure}[ht]
  \centering
  \begin{subfigure}[b]{0.45\textwidth}
    \includegraphics[width=\linewidth]{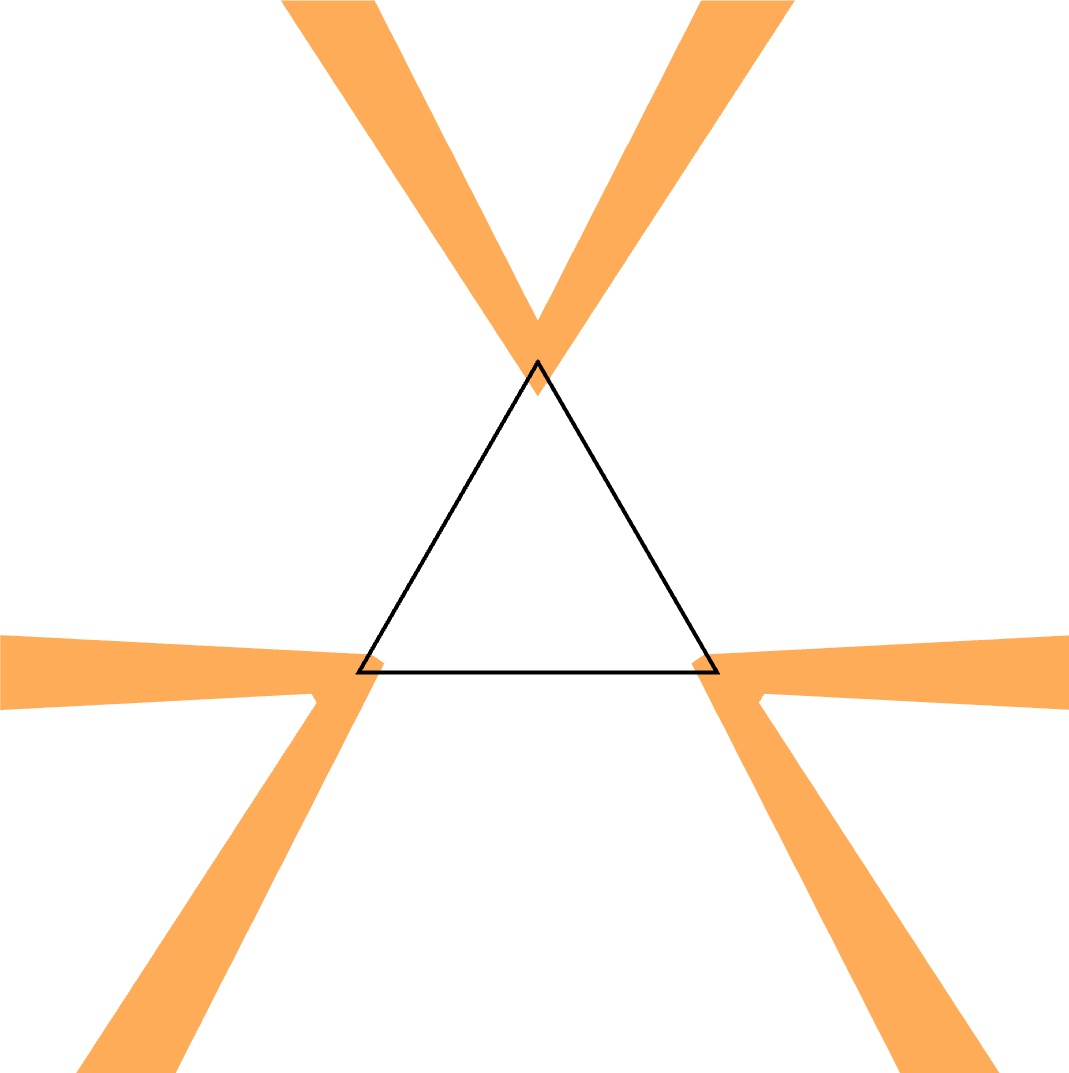}
    \caption{Two-dimensional case}
  \end{subfigure}
  \hfill
  \begin{subfigure}[b]{0.5\textwidth}
    \includegraphics[width=\linewidth]{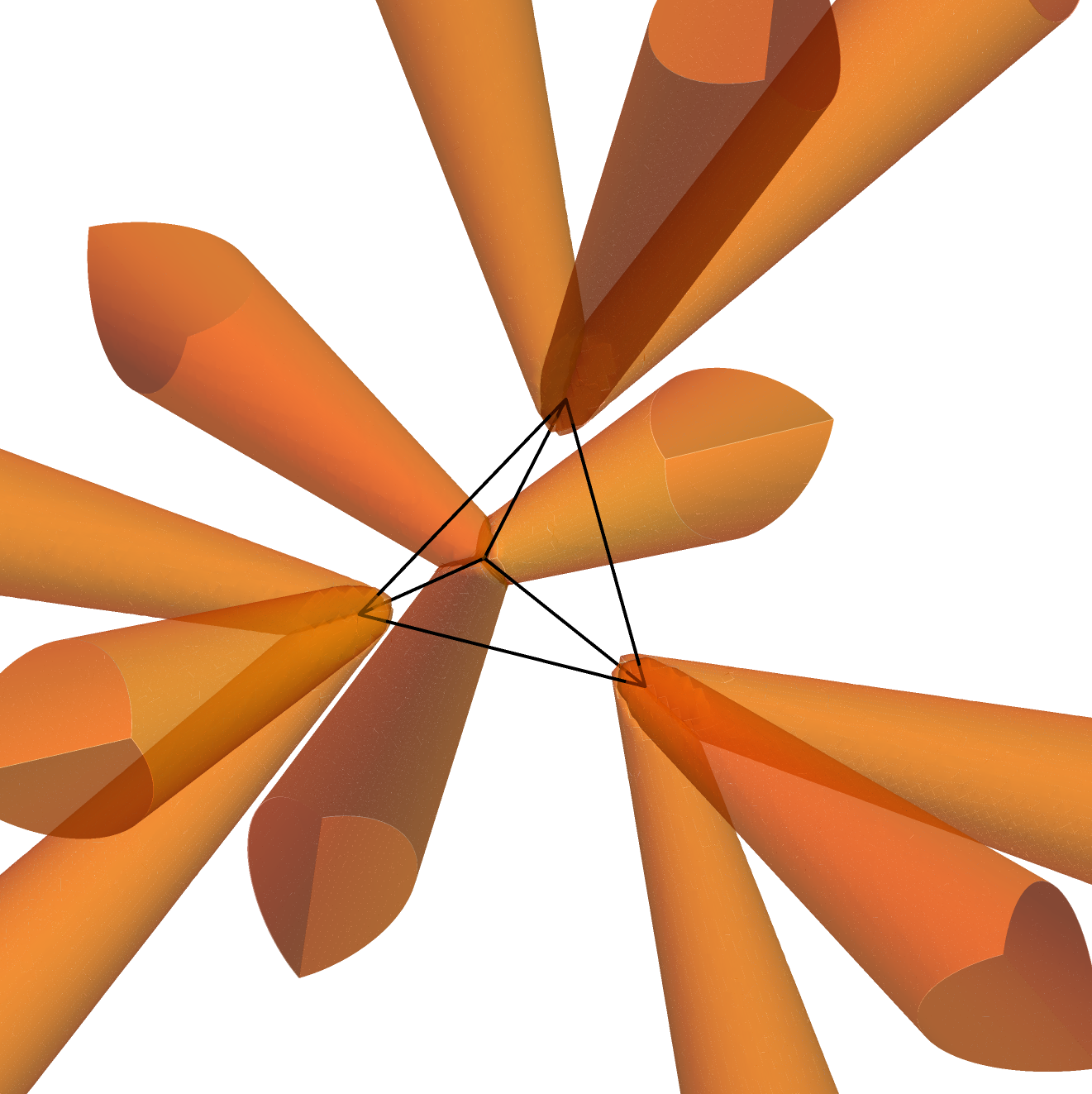}
    \caption{Three-dimensional case}
  \end{subfigure}
  \caption{All bad points lie in the orange region}
  \label{fig:criteria_visualization}
\end{figure}

The proof proceeds as follows. Assume that the point \(q\) satisfies  
\[
\angle BAq < \arccos\left( 1 - \frac{\alpha'}{\delta} \right).
\]  
Otherwise, by the preceding argument, the point \(q\) is \(\alpha'\)-good, and we can provide the response using Lemma~\ref{lem:criteria_good}.

An isometry \(\gamma\) is constructed so that the transformed point \(q' := \gamma(q)\) does not satisfy this property. Specifically, let \(A'\) and \(B'\) denote the farthest and nearest\footnote{We will even prove that the transformation preserves the identities of the nearest and farthest vertices.} vertices of \(\Delta\) with respect to \(q'\). The angle \(\angle B'A'q'\) then satisfies
\[
\angle B'A'q' \ge \arccos\left( 1 - \frac{\alpha'}{\delta} \right).
\] 

It is sufficient to construct such an isometry: if the point \(\gamma(q)\) is \(\alpha'\)-good with respect to the simplex \(\Delta\), then the original point \(q\) is \(\alpha'\)-good with respect to the transformed simplex \(\Delta' := \gamma^{-1}(\Delta)\). Once such a rotation is constructed, the remaining task is to demonstrate that \(\Delta'_{\alpha'} \subset \Delta_\alpha\).

Denote by \( a := \frac{\alpha}{2\delta} \) and \( b := \frac{\alpha'}{2\delta} \) the normalized neighborhood radii.

The main challenge in constructing such an isometry is that, if the transformation is not chosen carefully, for example, if the space is transformed solely to ensure \(\angle BAq' \ge \arccos\left( 1 - \frac{\alpha'}{\delta} \right)\), the point \(q'\) may still turn out to be bad. This can happen because the isometry might unintentionally change the identity of the farthest or nearest vertex of the simplex with respect to \(q'\).

This issue is resolved in Lemma~\ref{lem:rotate-ABQ}, where we construct a sufficiently small isometry \(R_{2\theta}\) for the specific query \(q\) such that it does not change the identity of the farthest or nearest vertex of the simplex with respect to~\(R_{2\theta}[q]\).

{Define \( \theta := \arccos(1 - 2b) \), and consider the rotation \( R_{2\theta} \), defined at the beginning of \cref{app:geom} (see Equation~\ref{eq:rotation}), associated with the query \(q\) and the simplex \(\Delta\).}
 Recall that this rotation acts on the plane \( \Pi := \mathrm{span}(A, B, Q) \), where \( Q \) is the centroid of the remaining vertices in \( \Delta \setminus \{A, B\} \), by rotating around the point \( A \) through an angle of \( 2\theta \). On the orthogonal complement \( \Pi^\perp \), it acts as the identity:
\(
R_{2\theta} \big|_{\Pi^\perp} = \mathrm{Id}_{\Pi^\perp}.
\)

Given the assumption that the point \(q\) is \(\alpha'\)-bad, it follows that
\[
\angle BAq < \arccos\left(1 - 2b\right) = \theta.
\]   

Observe that, under the assumption \(\alpha < \frac{\delta}{4}\), it follows that \(\alpha' < \frac{\delta}{32 \cdot 81}\). Therefore,
\[
1 - 2 \cdot \frac{\alpha'}{2\delta} = 1 - \frac{1}{32 \cdot 81} > \cos\left( \frac{\pi}{18} \right),
\]
which ensures that \(\theta < \frac{\pi}{18}\). This inequality can be verified using standard analysis tools, such as Taylor expansions of the cosine function.

Hence, we may apply Lemma~\ref{lem:rotate-ABQ}, which states that the farthest and nearest vertices with respect to 
\(q' := R_{2\theta}(q)\) are preserved under this isometry whenever \(\theta \le \tfrac{\pi}{18}\):
\[
  \arg\min_{A_i \in \Delta} \dist(A_i, q') = B, 
  \qquad 
  \arg\max_{A_i \in \Delta} \dist(A_i, q') = A.
\]

Moreover, \cref{lem:rotate-ABQ} also states that \(\angle q'AB > \theta\), and hence \(q'\) is \(\alpha'\)-good with respect to the simplex~\(\Delta\).

The remaining task is to demonstrate that \(\Delta'_{\alpha'} \subset \Delta_{\alpha}\). 

The largest displacement under the rotation \(R_{2\theta}^{-1}\) occurs in the plane \(ABQ\). 
Since all vertices of the simplex, except for \(A\), are equidistant from the origin \(A\), 
the point \(B\) is therefore the farthest from its image \(\gamma(B)\). 
Hence, to verify that \(\Delta'_{\alpha'} \subset \Delta_{\alpha}\), 
it suffices to demonstrate that
\[
  B(B', \alpha') \subset B(B, \alpha).
\]
Formally, this implication is proved in \cref{lem:neighborhoods_prereved_rotation}.

% By \cref{lem:neighborhoods_prereved_rotation}, it is enough to establish
% \[B(B', \alpha') \subset B(B, \alpha).\]
% Intuitively, it is enough to verify \(B(B', \alpha') \subset B(B, \alpha)\) because the largest displacement occurs in the plane \(ABQ\), since all vertices of the simplex, other than \(A\), are equidistant from the origin \(A\), thus the point \(B\) is the farthest from its image \(\gamma(B)\).

Verification that \(B(B', \alpha') \subset B(B, \alpha)\) follows from applying the Law of Cosines to the triangle \(\triangle B A R_{2\theta}(B)\), in which both \(\dist(A, B)\) and \(\dist(A, R_{2\theta}(B))\) are equal to \(2\delta\), and the angle at vertex \(A\) is \(2\theta\). The calculation yields:
\[
\dist(B, R_{2\theta}(B))^2 = 4\delta^2 + 4\delta^2 - 8\delta^2 \cos(2\theta) = 8\delta^2 (1 - \cos(2\theta)).
\]
Therefore,
\[
\dist(B, R_{2\theta}(B)) = 2\delta \sqrt{2(1 - \cos(2\theta))}.
\]

Thus, we require:
\[
\alpha - \alpha' \ge 2\delta \cdot \sqrt{2 - 2\cos(2\theta)} \quad \Longleftrightarrow \quad a - b \ge \sqrt{2(1 - \cos(2\theta))}.
\]

Applying the identity \(\cos(2\theta) = 2\cos^2\theta - 1 = 2(1 - 2b)^2 - 1\) yields:
\[
1 - \cos(2\theta) = 1 - \left[2(1 - 2b)^2 - 1\right] = 8b(1 - b).
\]

Thus, the condition becomes:
\[
a - b \ge \sqrt{16b(1 - b)}.
\]

Given that \(a > b\), both sides may be squared without altering the inequality:
\[
(a - b)^2 \ge 16b(1 - b)
\quad \Longleftrightarrow \quad
17b^2 - (16 + 2a)b + a^2 \ge 0.
\]

The discriminant of the quadratic polynomial \( f(b) = 17b^2 - (16 + 2a)b + a^2\) is given by
\[
\frac{D}{4} = (8 + a)^2 - 17a^2. 
\]

The inequality \( a - b \ge 4\sqrt{b-b^2} \) is satisfied only if \( a < \frac{(8+a) + \sqrt{(8+a)^2 - 17a^2}}{17} \), which occurs when
\[
0< b \le \frac{(8+a) - \sqrt{(8+a)^2 - 17a^2}}{17}.
\]

Finally, observe the following:
\[
\frac{(8+a) - \sqrt{(8+a)^2 - 17a^2}}{17} = \frac{a^2}{(8+a) + \sqrt{(8+a)^2 - 17a^2}} > \frac{a^2}{81} = b.
\]
This confirms the required condition and completes the argument.

\qed 

\bibliographystyle{abbrvnat}
\bibliography{bib}

\pagebreak

\appendix 
\section{Technical Results}
\label{appendix}

For completeness, this appendix includes several arguments that were deferred from the main text. Section \cref{app:right_continuity_e} establishes, via a topological approach, the right-continuity of the diameter-radius profile \(\mathtt e_X(\alpha)\) for totally bounded metric spaces. Section \cref{app:geom} introduces geometric techniques involving rotations of regular simplices, which are essential for the proof of Theorem~\ref{thm:pseudo-finiteness}. Finally, Section \cref{app:examples} contains detailed proofs of the illustrative examples introduced in \cref{sec:examples}.

%-----------------------------------------------------------
\subsection{Right-continuity of the diameter–radius profile}
\label{app:right_continuity_e}
%-----------------------------------------------------------

%-----------------------------------------------------------
%-----------------------------------------------------------
We start by fixing notation. 
Given a subset \(S\) in a metric space \((X, \dist)\), the \emph{Chebyshev radius} of~\(S\), denoted \(r(S)\), and the \emph{diameter} of \(S\), denoted \(\operatorname{diam} S\), are defined as follows:
\[
r(S) = \adjustlimits\inf_{q \in X} \sup_{x \in S} \dist(x, q), 
\qquad 
\operatorname{diam} S = \sup_{x, y \in S} \dist(x, y).
\]

The diameter-radius profile at \(\alpha\), denoted \(\mathtt e_X(\alpha)\), is the maximal radius of an enclosing ball among all subsets with diameter at most \(\alpha\). It is formally defined as follows:
\[
\mathtt{e}_X(\alpha) := \sup_{\substack{S \subseteq X \\ \diam(S) \le \alpha}} r(S).
\]
This section demonstrates that the function \(\mathtt e_X\) is right-continuous for every totally bounded metric space~\(X\).

Consider the set of compact subsets of the space \(X\), denoted by~\(\cK(X)\). 
Given two nonempty compact subsets \(A, B \subseteq X\), their \emph{Hausdorff distance} is defined as
\[
\mathrm{dist}_{\cK(X)}(A, B) := \max \left\{ \adjustlimits\sup_{a \in A} \inf_{b \in B} \dist (a, b), \, \adjustlimits\sup_{b \in B} \inf_{a \in A} \dist (a, b) \right\}.
\]

This number represents the smallest distance such that every point of one set lies within that distance of the other; in other words, it measures how far the two sets are from being contained in each other’s neighborhoods. The space \((\mathcal{K}(X), \mathrm{dist}_{\mathcal{K}(X)})\) is called the hyperspace of \(X\), where the distance is understood as the Hausdorff metric.

A classical result establishes that \(\mathcal{K}(X)\) is compact whenever \(X\) is compact (see Theorem 3.5 \cite{IllanesNadler1999}). Consequently, every sequence of compact subsets admits a subsequence that converges in the Hausdorff metric, which implies the right-continuity of \(\mathtt{e}_X(\alpha)\). The proof is provided below.

\begin{lemma}
\label{lem:right_continious_compact}
Let \(X\) be a compact metric space. Then the function \(\mathtt{e}_X(\alpha)\) is right-continuous. 
\end{lemma}
\begin{proof}
Any subset \(S \subset X\) has the same diameter and Chebyshev radius as its closure. Therefore, in the definition of the diameter-radius profile \(\mathtt{e}_X(\alpha)\), it is sufficient to consider closed subsets, which are compact due to the compactness of \(X\):
\[
\mathtt{e}_X\left((2+\epsilon)\delta\right) = \sup_{\substack{S \subset X \text{ compact} \\ \diam(S) \le (2+\epsilon)\delta}} r(S).
\]

The Chebyshev radius and the diameter are continuous functions on \(\cK(X)\). This follows from the inequalities below, which are valid for any compact subsets \(A\) and \(B\):
\[
| \diam \, A - \diam \, B | \le 2 \cdot \mathrm{dist}_{\cK(X)}(A, B), \qquad | r(A) - r(B) | \le \mathrm{dist}_{\cK(X)}(A, B).
\]
Therefore, these functions satisfy the Lipschitz condition, which guarantees their continuity.

Consider any decreasing sequence \(\beta_n \to \beta\). The function \(\mathtt{e}_X(\alpha)\) is non-decreasing, so the sequence \(\mathtt{e}_X(\beta_n)\) is non-increasing. To establish the result, it is necessary to show that \(\mathtt{e}_X(\beta_n) \to \mathtt{e}_X(\beta)\).

Suppose, for contradiction, that there exists \(\gamma > 0\) such that for every natural number \(n \in \NN\),
\[
\mathtt{e}_X(\beta_n) > \mathtt{e}_X(\beta) + \gamma.
\]

By the definition of \(\mathtt e_X\), for each natural number \(n\), there exists a set \(S_n\) with diameter at most \(\beta_n\) such that
\[
r(S_n) > \mathtt{e}_X(\beta) + \frac{\gamma}{2}.
\]

Since \(\cK(X)\) is compact, there exists a convergent subsequence of compact subsets \(S_{m(n)} \to S\) in the Hausdorff metric. The continuity of the diameter implies \(\diam \, S \le \beta\), and the continuity of the Chebyshev radius yields \(r(S_{m(n)}) \to r(S)\). Therefore,
\[
r(S) \ge \mathtt{e}_X(\beta) + \frac{\gamma}{2}.
\]
This contradicts the definition of \(\mathtt{e}_X(\beta)\), thereby completing the proof.
\end{proof}

\begin{lemma}
\label{lem:right_continious_totbounded}
For a totally bounded metric space \(X\), the function \(\mathtt e_X\) is right-continuous.
\end{lemma}
\begin{proof}
A well-known extension of the Heine–Borel theorem establishes that a metric space is compact
if and only if it is complete and totally bounded (see \cite{munkres2000topology}[Theorem 45.1]).  The completion \(\hat X\) of a totally bounded
space \(X\) is complete by construction and remains totally bounded, hence compact.
Since \(X\) is dense in \(\hat X\), every subset of \(X\) can be approximated
arbitrarily well by subsets of \(\hat X\) (and vice-versa), so \(\mathtt e_X=\mathtt e_{\hat X}\).
Lemma~\ref{lem:right_continious_compact} now applies to \(\hat X\), yielding the desired
right-continuity for \(\mathtt e_X\).
\end{proof}

%-----------------------------------------------------------
\subsection{Geometric tools for Euclidean simplices}\label{app:geom}
%-----------------------------------------------------------

%-----------------------------------------------------------
%-----------------------------------------------------------

This section presents the geometric constructions required for the proof of \cref{thm:pseudo-finiteness}, which are separated here to avoid interrupting the main line of argument.

Let \(\Delta = \{A_i\}_{i=0}^{n}\) denote a regular simplex\footnote{We use \( \Delta \) to denote the discrete set of \( n+1 \) vertices of the regular simplex, rather than its convex hull.} in \(\RR^n\) with edge length \(2\delta\). To specify two specific vertices, denote
\[
A := A_1,
\qquad
B := A_2.
\]
Let \(Q\) be the centroid of the remaining vertices:
\[
Q = \frac{1}{n-1} \sum_{C \in \Delta \setminus \{A, B\}} C.
\]

Consider the rotation
\[
R_{2\theta} : \RR^n \to \RR^n
\]
that fixes the vertex \(A\), acts in the plane \(\Pi = \mathrm{span}\{A, B, Q\}\) as a rotation by angle \(2\theta\) around \(A\) in the direction from \(\vvv{AB}\) to \(\vvv{AQ}\) along the smaller angle between them, and acts as the identity on the orthogonal complement \(\Pi^\perp\), which consists of all vectors orthogonal to \(\Pi\).

Place the point \(A\) at the origin and select the orthonormal basis \(\vec{d}_1, \ldots, \vec{d}_n\) such that:
\[
\begin{aligned}
\mathrm{span}\{\vec{d}_1, \vec{d}_2\} = \Pi, & & \vvv{AB} = k \cdot \vec{d}_1 \quad \text{with } k > 0, \\ 
& &\vvv{AQ} = k_1 \vec{d}_1 + k_2 \vec{d}_2 \quad \text{with } k_1, k_2 > 0.
\end{aligned}
\]
The rotation \(R_{2\theta}\) is represented in this basis\footnote{\textcolor{black}{Notice that such a basis exists because \(\angle BAQ \le \frac{\pi}{2}\).}} by the matrix
\begin{equation}
\label{eq:rotation}
R_{2\theta}
=
\begin{pmatrix}
\cos(2\theta) & -\sin(2\theta) & 0 & \cdots & 0 \\
\sin(2\theta) &  \cos(2\theta) & 0 & \cdots & 0 \\
0             &  0             & 1 &        & 0 \\
\vdots        &  \vdots        &   & \ddots &   \\
0             &  0             & 0 &        & 1
\end{pmatrix}.
\end{equation}

The upper-left \(2\times2\) block corresponds to a counterclockwise rotation in the plane \(\Pi\), and the rest acts as the identity on \(\Pi^\perp\).  

This transformation is a Euclidean isometry: it preserves all distances and acts as a rotation in the \(ABQ\)-plane while leaving the orthogonal directions unchanged.

\begin{remark}
\label{rem:cosines_notation}
Given two vectors \(v_1 := \vvv{NM}\) and \(v_2 := \vvv{NK}\) for some points \(\{N, M, K\}\), we will often refer to \(\cos \angle MNK\) as \(\cos(v_1, v_2)\). This emphasizes the computational role of the cosine as the inner product between the normalized vectors \(v_1\) and \(v_2\):
\[
\cos \angle MNK = \frac{v_1}{\|v_1\|} \cdot \frac{v_2}{\|v_2\|}.
\]
\end{remark}

The following lemma is a basic yet useful observation. 

\begin{lemma}
\label{lem:projection_cosine}
Let \(\Pi\) be a plane, and let \(A, B \in \Pi\) be two distinct points. Fix any point \(q \in \RR^n\), distinct from \(A\), and let \(H_q\) denote the orthogonal projection of \(q\) onto \(\Pi\). Then,
\[
\begin{cases}
0 < \cos\!\angle qAB \;\le\; \cos\!\angle H_qAB, & \text{if } \angle qAB < \tfrac\pi2,\\
0 > \cos\!\angle qAB \;\ge\; \cos\!\angle H_qAB, & \text{if } \angle qAB > \tfrac\pi2,\\
\cos\!\angle qAB = \cos\!\angle H_qAB = 0,   & \text{if } \angle qAB = \tfrac\pi2.
\end{cases}
\]
\end{lemma}

\begin{proof}
Since the cosine between two vectors can be computed via their inner product (see Remark~\ref{rem:cosines_notation}), the key observation is that  
\[
\vvv{AB} \cdot \vvv{Aq} = \vvv{AB} \cdot \vvv{AH_q}.
\]
This equality holds because the vectors \(\vvv{AB}\) and \(\vvv{H_q q}\) are orthogonal.
Indeed, since \(H_q\) is the projection of \(q\) onto the plane \(\Pi\), the vector \(\vvv{H_q q}\) is orthogonal to every vector in \(\Pi\), including \(\vvv{AB}\).

Hence,
\[
\vvv{AB} \cdot \vvv{Aq} = \vvv{AB} \cdot \vvv{AH_q}.
\]

If \(H_q \ne A\) (the case \(H_q = A\) is trivial), we compute:
\begin{align*}
\cos(\vvv{Aq}, \vvv{AB}) 
&= \frac{\vvv{Aq}}{\lVert Aq \rVert} \cdot \frac{\vvv{AB}}{\lVert AB \rVert} 
= \frac{\vvv{AH_q} \cdot \vvv{AB}}{\lVert Aq \rVert \cdot \lVert AB \rVert} \\
&= \left( \frac{\lVert AH_q \rVert}{\lVert Aq \rVert} \right)
\cdot \left( \frac{\vvv{AH_q}}{\lVert AH_q \rVert} \cdot \frac{\vvv{AB}}{\lVert AB \rVert} \right) \\
&= \cos(\vvv{AH_q}, \vvv{AB}) \cdot \frac{\lVert AH_q \rVert}{\lVert Aq \rVert}.
\end{align*}

Since \(AH_q\) is the projection of \(Aq\) onto the plane \(\Pi\), we have  
\[
0 \le \frac{\lVert AH_q \rVert}{\lVert Aq \rVert} \le 1,
\]
with equality only if \(q \in \Pi\).

Now observe that for any angle \(\phi \le \pi\),
\[
\cos \phi < 0 \quad \Longleftrightarrow \quad \phi > \frac{\pi}{2}.
\]
Therefore, the product \(\cos(\vvv{AH_q}, \vvv{AB}) \cdot \frac{\lVert AH_q \rVert}{\lVert Aq \rVert}\) is:
- smaller than \(\cos(\angle H_q AB)\) if \(\angle qAB < \frac{\pi}{2}\),
- larger if \(\angle qAB > \frac{\pi}{2}\), and
- equal when the angle is \(\frac{\pi}{2}\), as both cosines are zero in this case.

This completes the proof.
\end{proof}

\begin{lemma}[Rotation in the $ABQ$-plane keeps near/far order]
\label{lem:rotate-ABQ}
Assume for a query point \(q \in \RR^n\), the nearest and farthest points are \(B\) and \(A\) respectively; that is:
\[
B = \argmin_{A_i \in \Delta} \dist(q, A_i),
\qquad
A = \argmax_{A_i \in \Delta} \dist(q, A_i).
\]
Then, for any angle \(\theta < \frac{\pi}{18}\) whenever \(\angle BAq \le \theta\) the isometry \(R_{2\theta}\) preserves both the nearest and the farthest vertices of~\(\Delta\) with respect to \(q\):
\[
B = \argmin_{A_i \in \Delta} \dist(R_{2\theta} q, A_i),
\qquad
A = \argmax_{A_i \in \Delta} \dist(R_{2\theta} q, A_i).
\]

Moreover, the rotated point \(q' := R_{2\theta} q\) satisfies
\[
\angle q'AB > \theta.
\]
\end{lemma}

\begin{proof}

We will show that the nearest point \(B\) remains the nearest, and that the farthest point \(A\) likewise remains the farthest, whenever \(\angle qAB \le \frac{\pi}{18}\).
Finally, we demonstrate that \(\angle q'AB > \theta\).

To simplify our analysis, we choose an explicit orthonormal basis adapted to the geometry of the simplex. This choice simplifies computations and highlights geometric structure relevant to our problem.

\medskip
\noindent\textbf{Coordinates in an orthonormal basis.}

To simplify the calculations, we first scale the simplex \(A_1A_2\ldots A_{n+1}\) by a factor of \(1/(2\delta)\).
We then embed it in \(\RR^{n+1}\) by mapping the \(i^{\text{th}}\) vertex to \(\tfrac{1}{\sqrt2}e_i\).
Throughout, we set \(A := A_1\) and \(B := A_2\).
In these coordinates,
\[
\vvv{AB} = \tfrac1{\sqrt2}(-1,1,0,\ldots,0),
\qquad
\vvv{AQ} = \tfrac1{\sqrt2}\Bigl(-1,0,\tfrac1{n-1},\ldots,\tfrac1{n-1}\Bigr).
\]

To proceed, we introduce the two unit vectors that span the affine plane \(ABQ\); these vectors are key for defining the rotation \(R_{2\theta}\) (see~\eqref{eq:rotation}).

Writing the plane as
\[
ABQ = A + \langle \vec d_1,\vec d_2\rangle,
\]
with Minkowski addition and linear span, we take
\[
\vec d_1 := \vvv{AB},\qquad
\vec d_2 := \sqrt{\tfrac{n-1}{2(n+1)}}\Bigl(-1,-1,\tfrac{2}{n-1},\dots,\tfrac{2}{n-1}\Bigr).
\]
A direct check confirms that \(\{\vec d_1,\vec d_2\}\) is orthonormal.  Moreover
\[
2\,\vvv{AQ}-\vvv{AB}= \sqrt{\tfrac{n+1}{n-1}}\;\vec d_2,
\qquad\Longrightarrow\qquad
2\,\vvv{AQ}= \sqrt{\tfrac{n+1}{n-1}}\;\vec d_2+\vec d_1.
\]

\smallskip
\noindent\textit{Completing the basis.}  
Pick an orthonormal completion \(\{\vec d_i\}_{i=3}^n\) of \(\{\vec d_1,\vec d_2\}\).  
A convenient choice is
\[
\vec d_3:=\frac{\vvv{QA_3}}{\|QA_3\|}= 
\gamma\bigl(0,0,\tfrac{n-2}{n-1},-\tfrac1{n-1},\ldots,-\tfrac1{n-1}\bigr),
\quad
\gamma:=\sqrt{\tfrac{n-1}{n-2}},
\]
and analogous definitions for \(i\ge 4\).

With this basis, for any \(3 \le i \le n+1\) we have
\[
\vvv{AA_i}= 
\tfrac12\Bigl(
\sqrt{\tfrac{2(n-2)}{n-1}}\;\vec d_i+
\sqrt{\tfrac{n+1}{n-1}}\;\vec d_2+
\vec d_1
\Bigr).
\]

\medskip
\noindent\textbf{Coordinates of the query point.}
Let \(q\) satisfy \(\angle qAB\le\theta\) and write
\begin{equation}
\label{eq:query_coordinates}
\vvv{Aq}= \sum_{i=1}^{n} z_i\,\vec d_i,
\quad\text{so that}\quad
\frac{z_1}{\sqrt{\sum_{i=1}^{n}z_i^2}}=\cos(\angle qAB)\ge\cos\theta.
\end{equation}
In particular, \(z_1 \ge (\cos\theta / \sin\theta)\,|z_2|\).

Observe that for any vertex \(A_i\) of the simplex and for \emph{any} point \(C\), one has
\begin{equation}
\label{eq:distance_scalar_product}
|AC|\ge|A_iC|\;\Longleftrightarrow\;
2\,\vvv{AC}\cdot\vvv{AA_i}\;\ge\;1, 
\end{equation}
since 
\[
\|\vvv{A_iC} \|^2 = \|\vvv{AC}-\vvv{AA_i}\|^2 = \|\vvv{AC}\|^2 - 2 \vvv{AC} \cdot \vvv{AA_i}+1.
\]

Moreover, for any point \(C\) with coordinates 
\(\vvv{AC} = \sum_{i=1}^n y_i d_i\), 
a straightforward calculation shows that
\begin{equation}
\label{eq:scalar_simplex}
\vvv{AC}\cdot\vvv{AA_i}
=\tfrac12\Bigl(
\sqrt{\tfrac{2(n-2)}{n-1}}\,y_i
+\sqrt{\tfrac{n+1}{n-1}}\,y_2
+ y_1
\Bigr).
\end{equation}

\noindent\textbf{Step I: the farthest point remains the farthest when \(\angle qAB \le \frac{\pi}{8}\).}

Because \(R_{2\theta}\) acts as a planar rotation in \(\langle\vec d_1,\vec d_2\rangle\) and fixes the orthogonal complement, we have
\[
\vvv{Aq'}=
(\cos 2\theta\,z_1-\sin 2\theta\,z_2)\vec d_1
+(\cos 2\theta\,z_2+\sin 2\theta\,z_1)\vec d_2
+\sum_{i=3}^{n} z_i\,\vec d_i.
\]

Fix any \(n+1 \ge i \ge 3\). By \cref{eq:distance_scalar_product} and \cref{eq:scalar_simplex}, to show that \(|Aq| \ge |A_iq|\) implies \(|Aq'| \ge |A_iq'|\), it suffices to verify that the inequality
\[
\sqrt{\tfrac{2(n-2)}{n-1}}\,z_i
+\sqrt{\tfrac{n+1}{n-1}}\,z_2
+ z_1
\;\ge\;1
\]
implies
\[
\sqrt{\tfrac{2(n-2)}{n-1}}\,z_i
+\sqrt{\tfrac{n+1}{n-1}}\bigl(\cos 2\theta\,z_2+\sin 2\theta\,z_1\bigr)
+(\cos 2\theta\,z_1-\sin 2\theta\,z_2)
\;\ge\;1.
\]

\smallskip
For this implication to hold, it suffices to verify
\[
\bigl(\sqrt{\tfrac{n+1}{n-1}}\sin 2\theta +\cos 2\theta -1\bigr)\,z_1
\;\ge\;
\bigl(\sqrt{\tfrac{n+1}{n-1}}(1-\cos 2\theta)+\sin 2\theta\bigr)\,z_2,
\]
or, equivalently (using the double-angle identities), 
\[
(\sqrt{\tfrac{n+1}{n-1}}\cos\theta-\sin\theta)\,z_1
\;\ge\;
(\sqrt{\tfrac{n+1}{n-1}}\sin\theta+\cos\theta)\,z_2.
\]

Because \(z_1\ge(\cos\theta/\sin\theta)|z_2|\) (see \cref{eq:query_coordinates}), it is enough to verify
\[
\bigl(\sqrt{\tfrac{n+1}{n-1}}\cos\theta-\sin\theta\bigr)\,\tfrac{\cos\theta}{\sin\theta}
\;\ge\;
\sqrt{\tfrac{n+1}{n-1}}\sin\theta+\cos\theta,
\]
which in turn is equivalent to
\[
\sqrt{\tfrac{n+1}{n-1}}\cos 2\theta-\sin 2\theta\;\ge\;0.
\]
The latter holds for every \(0<\theta\le\tfrac\pi8\), completing the proof.

\medskip
\noindent\textbf{Step II: the nearest point remains the nearest when \(\angle qAB \le \frac{\pi}{18}\).}

We will prove that whenever \(\angle q'AB \le \frac{\pi}{6}\), all distances \(|A_iq'|\) for \(A_i \notin \{A, B\}\) are greater than \(|Bq'|\). To that end, let us denote the angle \(\angle q'AB\) by \(\phi\) and fix any \(n+1 \ge i \ge 3\).

By \cref{eq:distance_scalar_product}, the claim is equivalent to the scalar-product inequality
\[
\vvv{Aq'}\cdot \vvv{BA}\;\le\;\vvv{Aq'}\cdot \vvv{A_iA},
\]
and by \cref{eq:scalar_simplex}, in our orthonormal coordinates, this becomes
\[
\vvv{Aq'}\cdot \vvv{AB}=z_1
\;\ge\;
\tfrac12\Bigl(
   \sqrt{\tfrac{2(n-2)}{n-1}}\,z_i
   +\sqrt{\tfrac{n+1}{n-1}}\,z_2
   + z_1
\Bigr)  = \vvv{Aq'}\cdot \vvv{AA_i},
\]
or, equivalently,
\begin{equation}\label{eq:B-vs-A_3}
z_1
\;\ge\;
\sqrt{\tfrac{2(n-2)}{n-1}}\,z_i
+\sqrt{\tfrac{n+1}{n-1}}\,z_2.
\end{equation}

Assume \(\angle q'AB = \phi \le \pi/6\). Then, by \cref{eq:query_coordinates},
\[
z_1 \;\ge\; \frac{\cos\phi}{\sin\phi}\,\sqrt{z_2^2 + z_i^2}
          \;\ge\; \sqrt{3} \cdot \sqrt{z_2^2 + z_i^2},
\]
where the last inequality uses \(\frac{\cos \phi}{\sin \phi} \ge \frac{\cos(\pi/6)}{\sin(\pi/6)} = \sqrt{3}\), which holds for all \(0 < \phi \le \pi/6\).

We will now prove that
\[ 
3z_2^2 + 3z_i^2 \;\ge\; \left( \sqrt{\tfrac{2(n-2)}{n-1}}\,z_i
+ \sqrt{\tfrac{n+1}{n-1}}\,z_2 \right)^2.
\]
Expanding the right-hand side, this is equivalent to
\[ 
\frac{n+1}{n-1} z_i^2 + \frac{2(n-2)}{n-1} z_2^2 
\;\ge\; 
2 \sqrt{ \frac{2(n-2)(n+1)}{(n-1)^2} }\, z_i z_2,
\]
which is always true by the inequality \(U^2 + S^2 \ge 2US\), applied with
\[
U = \sqrt{\frac{n+1}{n-1}}\,z_i,
\qquad
S = \sqrt{\frac{2(n-2)}{n-1}}\,z_2.
\]

Putting these estimates together, we obtain
\[
z_1 \;\ge\; \sqrt{3z_2^2 + 3z_i^2}
          \;\ge\; \sqrt{\tfrac{2(n-2)}{n-1}}\,z_i
                 + \sqrt{\tfrac{n+1}{n-1}}\,z_2,
\]
which is exactly \eqref{eq:B-vs-A_3}. Hence \(\lvert Bq' \rvert \le \lvert A_iq' \rvert\) whenever \(\angle q'AB \le \pi/6\).

\smallskip
\noindent\textit{Consequence for the rotated point.}
If the rotation parameter satisfies \(\theta \le \pi/18\), then the rotated point \(q'\) obeys
\(\angle BAq' \le \pi/6\); therefore the point \(B\) remains the nearest after rotation.

Hence, if the angle \(\theta \le \frac{\pi}{18}\), the angle \(\angle BAq' \le \frac{\pi}{6}\) and the point \(B\) still remains the nearest. 

\noindent\textbf{Step III: proving \(\angle q'AB > \theta\)}

It remains to verify that the rotation \(R_{2\theta}\) sufficiently increases the angle; specifically, we need to confirm that \(\angle BAq' \ge \theta\).

To show that decompose the vector \(\vvv{Aq}\) as
\[
\vvv{Aq} = \vvv{AH_q} + \vvv{H_q q},
\]
where \(H_q\) is the projection of \(q\) onto the plane \(\Pi\). Denote \(R_{2 \theta}[H_q]\) by \(H_q'\). 
Since \(R_{2 \theta}\) is not only an isometry, but by construction also a linear  transformation with placing the point \(A\) as the origin: 
\[
\begin{aligned}
\vvv{Aq'} = R_{2 \theta}[\vvv{Aq}] = R_{2 \theta}[\vvv{AH_q}] + R_{2 \theta}[\vvv{H_q q}]= \\
\vvv{AH_q'} + R_{2 \theta}[\vvv{H_q q}]. 
\end{aligned}
\]
Since the rotation \(R_{2 \theta}\) preserves the vectors orthogonal to \(\Pi\), it also preserves \(\vvv{H_q q}\). 
Hence (see \cref{fig:rotation-decomposition}):
\[ 
\vvv{Aq'} = \vvv{AH_q'} + \vvv{H_qq}, 
\]
and moreover, \(H_q'\) is the projection of \(q'\) onto the plane \(\Pi\).

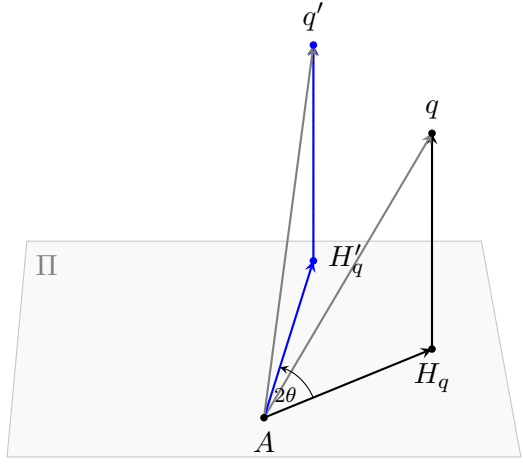
\begin{figure}[ht]
  \centering
\begin{tikzpicture}[scale=2.6,>=stealth,
                    dot/.style={circle,fill,inner sep=1pt},
                    plane/.style={gray!18,fill opacity=.25}]

  %--- the rotation plane Π ----------------------------------------
  \coordinate (P1) at (-1.3,-0.2);
  \coordinate (P2) at ( 1.3,-0.2);
  \coordinate (P3) at ( 1.1, 0.9);
  \coordinate (P4) at (-1.2, 0.9);
  \fill[plane] (P1)--(P2)--(P3)--(P4)--cycle;       % shaded plane
  \draw[gray!45] (P1)--(P2)--(P3)--(P4)--cycle;

  %--- key points --------------------------------------------------
  \coordinate (A)   at (0,0);        % choose A as origin for clarity
  \coordinate (Hq)  at (0.85,0.35);  % projection of q   on Π
  \coordinate (Hqp) at (0.25,0.80);  % rotated projection H'_q
  \coordinate (q)   at ($(Hq)+(0,1.1)$);   % q lies "above" Π
  \coordinate (qp)  at ($(Hqp)+(0,1.1)$);  % q' above Π

  %--- decomposition vectors before rotation ----------------------
  \draw[->,thick] (A)  -- (Hq);
  \draw[->,thick] (Hq) -- (q);
  \draw[->,thick,gray] (A)  -- (q);

  %--- after rotation ---------------------------------------------
  \draw[->,thick,blue]
        (A)  -- (Hqp);
  \draw[->,thick,blue]
        (Hqp) -- (qp);
  \draw[->,thick,gray]
        (A)  -- (qp);]

  %--- vertical guidelines (dotted, to stress orthogonality) -------
  \draw[dotted] (Hq) -- (q);
  \draw[dotted] (Hqp) -- (qp);

  %--- dots & labels ----------------------------------------------
  \node[dot,label={below:$A$}] at (A) {};
  \node[dot,label={below:$H_q$}] at (Hq) {};
  \node[dot,blue,label={right:$H'_q$}] at (Hqp) {};
  \node[dot,label={above:$q$}] at (q) {};
  \node[dot,blue,label={above:$q'$}] at (qp) {};
  \node[gray] at (-1.1,0.78) {$\Pi$};
  \pic[draw,->,angle radius=7mm,
       "$2\theta$"{font=\scriptsize}] {angle = Hq--A--Hqp};
\end{tikzpicture}
\caption{Rotation and vector decomposition of $\vvv{Aq}$.}
  \label{fig:rotation-decomposition}
\end{figure}

To see that \(\angle q'AB \ge \theta\), note that the cosine function is decreasing on the interval \([0, \pi]\). Therefore, it suffices to show that
\[
  \cos(\angle q'AB) \le \cos\theta.
\]

For this purpose, we will use Lemma~\ref{lem:projection_cosine},  which formalizes the observation that projecting a point onto a plane either increases or decreases the cosine of an angle, 
depending on whether the angle is acute or obtuse.
Specifically, if \(\angle qAB < \frac{\pi}{2}\), then \(\cos \angle qAB \ge \cos \angle H_qAB\), and the inequality is reversed when the angle is obtuse.\footnote{It also shows that the projected angle is acute or obtuse if and only if the original angle was acute or obtuse.} 
However, to apply this lemma correctly, one must verify that either the original angle or its projection is acute or obtuse, as the conclusion depends on this distinction.

By Lemma~\ref{lem:projection_cosine}, we have \(\cos \angle H_qAB \ge \cos \angle qAB\), since \(\angle qAB \le \theta < \frac{\pi}{2}\). Therefore, using again that cosine is decreasing on the interval \([0, \pi]\), it follows that \(\angle H_qAB \le \theta\).

Using the triangle inequality for angles, and the fact that \(\angle H_qAH_q' = 2\theta\), we deduce:
\[
\begin{aligned}
\angle H_q'AB &\le \angle H_qAH_q' + \angle H_qAB \le 2\theta + \theta = 3\theta < \frac{\pi}{2}, \\
\angle H_q'AB &\ge \angle H_qAH_q' - \angle H_qAB \ge 2\theta - \theta = \theta.
\end{aligned}
\]

Therefore, applying Lemma~\ref{lem:projection_cosine} once again, we conclude:
\[
\cos \angle q'AB \le \cos \angle H_q'AB \le \cos \theta.
\]
\end{proof}

\begin{lemma}
\label{lem:neighborhoods_prereved_rotation}
    Suppose \(\alpha', \alpha > 0\) are such that 
    \(R^{-1}_{2\theta}\big[B(B, \alpha')\big] \subset B(B, \alpha)\).
    Then 
    \[
        R^{-1}_{2\theta}\big[\Delta_{\alpha'}\big] \subset \Delta_\alpha.
    \]
\end{lemma}
\begin{proof}
Denote the simplex \(R^{-1}_{2\theta}\big[\Delta\big]\) by \(\Delta'\).
To ensure the inclusion \( \Delta'_{\alpha'} \subset \Delta_\alpha \), it suffices to check that no vertex of the simplex moves by more than \( \alpha - \alpha' \) under the rotation \(R_{2\theta}\). 

Let \(C\) be any vertex distinct from \(A\) and \(B\). Then the point \(C\) does not lie in the affine plane \(ABQ\).
Denote by \(B' := R_{2 \theta} B\), by \(C' := R_{2 \theta} C\), and by \(Q' := R_{2 \theta} Q\). The point \(C\) is projected onto~\(Q\).\footnote{But this fact is not essential; the argument still holds without requiring that the projection coincides with \(Q\).}
Then,
\[
\cos(\vvv{AQ}, \vvv{AQ'}) = \frac{\vvv{AQ} \cdot \vvv{AQ'}}{\|AQ\|^2} = \cos 2\theta,
\]
due to the construction of the rotation.
On the other hand, since \( \vvv{AQ} \perp \vvv{QC} \), one has:
\[
\cos(\vvv{AC}, \vvv{AC'}) = \frac{\vvv{AQ} \cdot \vvv{AQ'} + \|CQ\|^2}{\|AQ\|^2 + \|CQ\|^2}
\;\ge\;
\frac{\vvv{AQ} \cdot \vvv{AQ'}}{\|AQ\|^2} = \cos 2\theta,
\]
since adding the same positive value to both the numerator and denominator of a ratio in \((0, 1]\) does not decrease the ratio.
Now,
\[
\begin{aligned}
\|BB'\|^2 - \|CC'\|^2 &= \|\vvv{AB} - \vvv{AB'}\|^2 - \|\vvv{AC} - \vvv{AC'}\|^2 \\
&= 2\left( \vvv{AC} \cdot \vvv{AC'} - \vvv{AB} \cdot \vvv{AB'} \right) \\
&= 2\|\vvv{AB}\|^2 \left( \cos(\vvv{AC}, \vvv{AC'}) - \cos 2\theta \right) > 0,
\end{aligned}
\]
which shows that indeed \(\|CC'\| < \|BB'\|\). 
\end{proof}

%-----------------------------------------------------------
\subsection{Examples}\label{app:examples}
%-----------------------------------------------------------

We begin by analyzing the game on the real line.

\begingroup
\begin{example*}[\Cref{ex:unbounded-real} from the Introduction]
Let \(\epsilon > 0\) and \(\delta \ge 0\). Then, for any number \(T\) of queries,
\[
\OPT_{\RR}(T, \epsilon, \delta) = +\infty.
\]
\end{example*}
\endgroup

\begin{proof}
The idea of the proof is straightforward: since the space is unbounded, the responder can—already in the first round—return an arbitrarily large answer. This ensures that the initial feasible region is as large as desired. Then, over the course of \(T\) interactions, the responder can control how fast the region shrinks, ensuring that the final feasible region remains arbitrarily large.

Let us elaborate.

At the start of the game, for any large number \(L_0 > 0\) and any query \(q \in \RR\), the responder may answer with
\[
r_q := \frac{1+\epsilon}{(1+\epsilon)^2 - 1} \left(L_0 - (2+\epsilon)\delta \right),
\]
which results in a feasible region that includes two intervals of length \(L_0\).

Now fix an interval \([a, b]\) of length \(L\), and suppose the reconstructor asks a query \(q \in \RR\). The responder then answers with
\[
r_q := \frac{\max\{|q - b|,\, |q - a|\} - \delta}{1 + \epsilon}.
\]
This response places the point in \([a, b]\) that is farthest from \(q\) right on the boundary of the feasible region \(\Phi(q, r_q)\). 
In particular, this implies that every point \(x \in [a,b]\) satisfies \(\lvert x - q\rvert \leq (1+\epsilon)r_q+\delta\).

Assume without loss of generality that \(\max\{|q - b|,\, |q - a|\} = |q - b|\), i.e., \(q \le \frac{a + b}{2}\).
On the other hand, all points \(x \in [a, b]\) satisfying
\[
\lvert x - q\rvert  \ge \frac{\lvert b - q\rvert - (2 + \epsilon)\delta}{(1+\epsilon)^2}
\]
also satisfy \((1+\epsilon) \cdot \lvert x - q\rvert + \delta \geq r_q\).
Thus, all such points lie within \(\Phi(q, r_q)\).

The length of the subinterval of \([a,b]\) consisting of such points is
\[
\frac{((1+\epsilon)^2 - 1) |b - q| + (2 + \epsilon)\delta}{(1+\epsilon)^2} \ge \frac{((1+\epsilon)^2 - 1) \cdot \lvert b - a\rvert}{2(1+\epsilon)^2} = \frac{( (1+\epsilon)^2 - 1 )}{2(1+\epsilon)^2} \cdot \lvert b - a\rvert.
\]
Hence, on each round, the responder can reduce the feasible region’s length by a constant multiplicative factor \(c := \frac{(1+\epsilon)^2 - 1}{2(1+\epsilon)^2}\). Starting from an interval of arbitrary length \(L_0\), the feasible region after \(T\) rounds can still have length at least \(c^T \cdot L_0\), which diverges as \(L_0 \to \infty\).

Therefore,
\[
\OPT_{\RR}(T, \epsilon, \delta) = +\infty.
\]
\end{proof}
The next example demonstrates that when $\epsilon=0$, the real line is $(\epsilon,\delta)-$pseudo-finite for every~$\delta>0$:
\begin{example}[Pseudo‑finiteness on the real line]\label{ex:R-pseudo}
For every $\delta\ge 0$, the real line $\RR$ with its usual metric is $(0,\delta)$‑pseudo‑finite.  
\end{example}

\begin{proof}
The optimal reconstructor strategy is to ask two query points $q_1,q_2\in\RR$ with $q_2-q_1>2\delta$. Let the answers of the responder be $r_1,r_2$. 

Intuitively, each answer restricts the secret to intervals of length $2\delta$ centered at $q_i\pm r_i$. Because the distance between the center of the leftmost interval and the center of the rightmost interval exceeds \(2\delta\), at most two of the four candidate intervals overlap, and their intersection has diameter \(2\delta\), attaining the optimal error (see Figure~\ref{fig:overlap-intervals}).

Formally, the feasible regions are
\[
\begin{aligned}
\Phi(q_1, r_1) &= B\left(q_1 - r_1, \delta\right) \cup B\left(q_1 + r_1, \delta\right), \\
\Phi(q_2, r_2) &= B\left(q_2 - r_2, \delta\right) \cup B\left(q_2 + r_2, \delta\right).
\end{aligned}
\]

Assume there are two points \(x, y\) in the intersection \(\Phi(q_1, r_1) \cap \Phi(q_2, r_2)\), such that \(y - x > 2\delta\).  
These points cannot lie in the same ball of radius \(\delta\), hence
\[
x \in B\left(q_1 - r_1, \delta\right), \quad y \in B\left(q_1 + r_1, \delta\right),
\]
and also
\[
x \in B\left(q_2 - r_2, \delta\right), \quad y \in B\left(q_2 + r_2, \delta\right).
\]

Therefore, the balls with larger and smaller centers must overlap:
\[
|q_2 + r_2 - (q_1 + r_1)| < 2\delta, \quad \text{and} \quad |q_2 - r_2 - (q_1 - r_1)| < 2\delta.
\]

On the other hand,
\[
|q_2 - q_1 + r_2 - r_1| + |q_2 - q_1 + r_1 - r_2| \ge 2|q_2 - q_1| > 4\delta,
\]
which leads to a contradiction. Hence the result.
\end{proof}

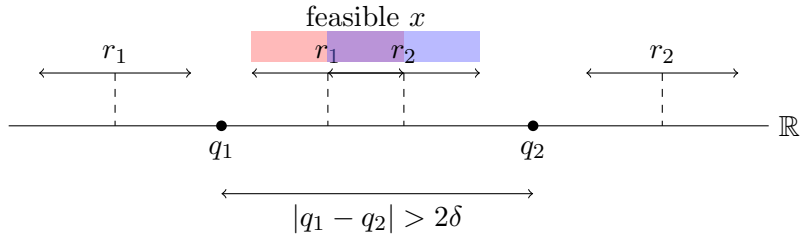
\begin{figure}[ht]
\centering
\begin{tikzpicture}[x=1cm,y=1cm]

  % real line ---------------------------------------------------
  \draw (-5,0) -- (5,0) node[right] {$\mathbb R$};

  % query points q1, q2 ----------------------------------------
  \fill (-2.2,0) circle (2pt) node[below=2pt] {$q_1$};
  \fill ( 1.9,0) circle (2pt) node[below=2pt] {$q_2$};

  % annotation |q1 - q2| > 2δ ---------------------------------
  \draw[<->] (-2.2,-0.9) -- (1.9,-0.9)
        node[midway,below] {$|q_1-q_2|>2\delta$};

  % ----- four length‑2δ candidate intervals -------------------

  % leftmost (q1 - r1) – just dashed, no colour
  \coordinate (L1) at (-3.6,0.7);
  \draw[dashed] (L1) -- ++(0,-0.7);
  \draw[<->] ($(L1)+(-1,0)$) -- ($(L1)+(1,0)$);
  \node[above] at (L1) {$r_1$};

  % centre‑left (q1 + r1) – RED, half‑length δ = 1cm
  \coordinate (C1) at (-0.8,0.7);
  \draw[dashed] (C1) -- ++(0,-0.7);
  %     coloured bar: [−1.8 ,0.2]
  \fill[red!60,fill opacity=0.45] ($(C1)+(-1,0.15)$) rectangle ($(C1)+(1,0.55)$);
  \draw[<->] ($(C1)+(-1,0)$) -- ($(C1)+(1,0)$);
  \node[above] at (C1) {$r_1$};

  % centre‑right (q2 − r2) – BLUE, same length, centre 0.2
  \coordinate (C2) at (0.2,0.7);
  \draw[dashed] (C2) -- ++(0,-0.7);
  % coloured bar: [−0.8 ,1.2]
  \fill[blue!60, fill opacity=0.45] ($(C2)+(-1,0.15)$) rectangle ($(C2)+(1,0.55)$);
  \draw[<->] ($(C2)+(-1,0)$) -- ($(C2)+(1,0)$);
  \node[above] at (C2) {$r_2$};

  % rightmost (q2 + r2) – dashed, no colour
  \coordinate (R2) at (3.6,0.7);
  \draw[dashed] (R2) -- ++(0,-0.7);
  \draw[<->] ($(R2)+(-1,0)$) -- ($(R2)+(1,0)$);
  \node[above] at (R2) {$r_2$};

  % label for overlap (purple region appears automatically)
  \node[above=4pt] at (-0.3,1.05) {feasible $x$};

\end{tikzpicture}
\caption{With \(|q_1-q_2|>2\delta\), only the red interval
\([q_1+r_1-\delta,\,q_1+r_1+\delta]\) and the blue interval
\([q_2-r_2-\delta,\,q_2-r_2+\delta]\) intersect, pinning the secret point
to their (purple) overlap.}
\label{fig:overlap-intervals}
\end{figure}

The following simple observation shows that in bounded metric spaces, the reconstruction game becomes trivial whenever \( (2+\epsilon)\delta \) exceeds the diameter of the space.

\begin{example}
Any bounded metric space \( X \) with \( \diam(X) \leq (2+\epsilon)\delta \) is \((\epsilon, \delta)\)-pseudo-finite. Indeed, in this regime, the responder can maintain the entire space as feasible throughout the interaction by consistently replying with the constant value \( \delta \). As a result, the optimal reconstruction error is simply the Chebyshev radius of \( X \), which the reconstructor can achieve without submitting any queries.
\end{example}
% The following simple observation illustrates that one should put some constraints on $(\epsilon, \delta)$, since for big enough $(\epsilon, \delta)$ the responder might answer consistently to \emph{all} points in $X$. 
% \begin{example}
% Any compact metric space X of diameter less than $(2+\epsilon)\delta$ is \((\epsilon, \delta)\) pseudo-finite. 
% \end{example}

% However, even for seemingly trivial spaces, pseudo‑finiteness is \emph{not}
% a monotone in the parameters $(\epsilon,\delta)$.

% Any point $x\in\RR^{n}$ is uniquely determined by its distances to the
% $n+1$ vertices of a non‑degenerate $n$‑simplex\;\cite[§2]{Blumenthal1953}.
% Hence:

The next two examples concern noiseless responders (i.e., \( \epsilon = \delta = 0 \)):

\begin{example}
The Euclidean space \( \mathbb{R}^n \) is \((0,0)\)-pseudo-finite.  
Indeed, any point \( x \in \mathbb{R}^n \) is uniquely determined by its distances to the
\( n+1 \) vertices of a non-degenerate \( n \)-simplex~\cite[§2]{Blumenthal1953}.
\end{example}
{The same holds for any subset of \(\mathbb{R}^n\) that contains such a simplex. However, even in the noiseless setting \((\epsilon, \delta) = (0,0)\), pseudo-finiteness does not hold in all metric spaces—even if the space is totally bounded:}
\begin{example}
\label{ex:binary_sequences}
Let \( X = \{0,1\}^{\mathbb{N}} \) be the space of infinite binary sequences, equipped with the standard ultrametric: the distance between two sequences \( \alpha = (\alpha_i)_{i \in \mathbb{N}} \) and \( \beta = (\beta_i)_{i \in \mathbb{N}} \) is defined as \( d(\alpha, \beta) = 2^{-j} \), where \( j \) is the first index for which \( \alpha_j \neq \beta_j \). Then \( X \) is a compact (and hence totally bounded) metric space that is not \((0,0)\)-pseudo-finite.
\end{example}

\begin{proof}
We show that \( \OPT_X(T, 0, 0) \ge 2^{-T-1} \) for every \( T \), by explicitly constructing a responder strategy. The goal is to preserve a feasible set of sequences that agree on at most \( T \) coordinates.

Assume that after round \( t \), the responder has committed to at most the first \( t' \le t \) bits of the secret sequence. Given a query \( q = (q_i)_{i \in \mathbb{N}} \), the responder replies as follows:

\begin{itemize}
    \item If the prefix \( (q_1, \dots, q_{t'}) \) disagrees with the committed prefix, respond with the true distance \( 2^{-j} \), where \( j \) is the first index of disagreement.
    \item Otherwise, respond with \( r = 2^{-t'-1} \), and define the next bit of the secret sequence as \( \alpha_{t'+1} := 1 - q_{t'+1} \).
\end{itemize}

After \( T \) rounds, the responder has specified exactly \( T \) bits. Let the reconstructor return a sequence~\( \hat{x} \). Then the responder chooses a secret point \( x^\star \) that agrees with~\( \hat{x} \) on all bits except for bit \( T+1 \), which is flipped. This implies that \( \dist(\hat{x}, x^\star) = 2^{-T-1} \), yielding the lower bound.

\smallskip
\noindent
\emph{Remark:} One can further show that this lower bound is tight, and that \( \OPT_X(T, 0, 0) = 2^{-T-1} \), since every informative query forces the responder to reveal one additional bit.
\end{proof}

We now present an example of a non–totally bounded metric space for which the diameter-radius profile \(\mathtt e_X\) fails to be right-continuous.

\begin{example}[Failure of right-continuity of \(\mathtt e_X\)]
\label{ex:not_right_eX}
Recall that for a metric space \((X, \dist)\) the function \(\mathtt e_X\) is defined by
\[
\mathtt e_X(\alpha)\;:=\;\sup\{\, r(S)\;:\; S\subseteq X,\ \diam(S)\le \alpha \,\},
\]
where the Chebyshev radius and diameter are
\[
r(S)\;:=\;\inf_{q\in X}\ \sup_{x\in S}\dist(x,q),
\qquad
\diam(S)\;:=\;\sup_{x,y\in S}\dist(x,y).
\]

Let \(X=\{x_n,y_n: n\in\NN\}\) with metric
\[
\dist(x_n,y_n)=1+\tfrac1n\quad\text{for each }n,\qquad
\dist(u,v)=2\ \text{for all other distinct }u\neq v.
\]
Then \(\mathtt e_X\) is \emph{not} right-continuous\footnote{The space \(X\) is not totally bounded: for \(\alpha\le 1\) every \(\alpha\)-ball contains at most one point (all nonzero distances in \(X\) exceed \(1\)), so no finite \(\alpha\)-net exists; equivalently, the only \(\alpha\)-cover is \(X\) itself, which is infinite.}. 
\end{example}

\begin{proof}
If \(\alpha\le 1\), then any subset \(S\subseteq X\) with \(\diam(S)\le\alpha\) must be a singleton (since every nontrivial distance is \(>1\)), hence \(\mathtt e_X(\alpha)=0\).

For each \(n\), let \(S_n=\{x_n,y_n\}\). Then \(\diam(S_n)=1+\tfrac1n\). Moreover,
\[
r(S_n)\;=\;\inf_{q\in X}\ \max\{\dist(x_n,q),\dist(y_n,q)\}
=\min\{\,1+\tfrac1n,\ 2\,\}=1+\tfrac1n,
\]
because choosing \(q\in\{x_n,y_n\}\) yields value \(1+\tfrac1n\), while any \(q\notin S_n\) is at distance \(2\) from both points. 

Note that any subset of \(X\) with at least three distinct points contains two points at distance \(2\), hence has diameter \(2\). Therefore, for \(1+\tfrac1n\le \alpha < 1+\tfrac1{n-1}\) the only nontrivial subsets with \(\diam\le \alpha\) are the pairs \(S_k\) with \(k\ge n\), and thus
\[
\mathtt e_X(\alpha)=\max_{k\ge n} r(S_k)=1+\tfrac1n.
\]
Consequently,
\[
\lim_{\alpha\downarrow 1}\mathtt e_X(\alpha)
=\lim_{n\to\infty}\mathtt e_X\!\left(1+\tfrac1n\right)=1,
\qquad\text{while}\qquad
\mathtt e_X(1)=0,
\]
so \(\mathtt e_X\) is not right-continuous at \(\alpha=1\).
\end{proof}

We conclude the section by formally proving the equivalence between the Dinur–Nissim model and the reconstruction game on the Boolean cube as referenced in \cref{ex:DinurNissim_equivalence}.

\begin{example}[Dinur–Nissim model]
\label{ex:app_DinurNissim_equivalence}
The counting-query game in the Dinur–Nissim model is equivalent to the distance-based game on the Boolean cube with the Hamming metric, namely, every query in one game can be simulated by at most two queries in the other. 
\end{example}
\begin{proof}
    We show that the counting-query game is equivalent to the distance-based game on the Boolean cube (with Hamming distance) by introducing an intermediate step: both games are equivalent to an inner-product game played on \( \{\pm 1\}^n\). 
    
    The inner-product game is defined as follows. The responder chooses a secret vector \( D' = (d_1', \ldots, d_n') \in \{\pm 1\}^n \). In each round, the reconstructor submits a query vector \[ w = (w_1, \ldots, w_n) \in \{\pm 1\}^n, \] and the responder replies with a noisy approximation of the inner product
    \[
    \langle D', w \rangle = \sum_{i=1}^n w_i d_i'.
    \]
    
    \noindent\textbf{Step I: From the Dinur–Nissim model to the inner-product game.} In the Dinur–Nissim model, the dataset is a binary vector \(D = (d_1, \ldots, d_n) \in \{0,1\}^n\), and each query is a subset \(q \subseteq [n]\), whose (noisy) answer is the count \[a_q = \sum_{i \in q} d_i.\] 
    We can represent the subset \(q\) by its indicator vector \(v_q \in \{0,1\}^n\), so that \(a_q = \langle D, v_q \rangle\). 
    To simulate this count using the inner-product game on \(\{\pm 1\}^n\), consider the transformation 
    \[v \mapsto 2v - 1,\] 
    which maps \(\{0,1\}^n\) to \(\{\pm 1\}^n\). 
    Let 
    \[D’ = 2D - 1 \quad \text{and} \quad w_q = 2v_q - 1.\]
    Then we have the identity 
    \[\langle D’, w_q \rangle = 4 \langle D, v_q \rangle - 2 \langle D, \mathbf{1} \rangle - 2|q| + n.\]
    Therefore, we can recover the original count \(\langle D, v_q \rangle\) by submitting two inner-product queries: one with \(w_q\) and one with the all-ones vector \(\mathbf{1}\).
    A similar argument gives the reverse direction. 
    
    \noindent\textbf{Step II: From the inner-product game to the distance-based game.} 
    Next, we show that the inner-product game on \(\{\pm 1\}^n\) is equivalent to the distance game on \(\{\pm 1\}^n\) equipped with the Hamming metric.
    On this space, one has the identity 
    \[\mathrm{dist}_{\text{Ham}}(x, y) = \frac{1}{4} \|x - y\|_2^2 = \frac{n}{2} - \frac{1}{2} \langle x, y \rangle.\]

    Hence, given the inner product \(\langle x, y \rangle\) one can recover the Hamming distance, and conversely, via simple affine transformations. This correspondence between the models also modifies the noise parameters, but only in a controlled manner. Since the simulation uses at most two queries and involves only affine transformations, the noise in the simulated model increases by at most a constant multiplicative factor.

\end{proof}

\end{document}